\documentclass[graybox]{s-svmult}

\usepackage{mathptmx}       
\usepackage{helvet}         
\usepackage{courier}        
\usepackage{type1cm}        
\usepackage{makeidx}         
\usepackage{graphicx}        
\usepackage{multicol}        
\usepackage[bottom]{footmisc}

\usepackage{booktabs}  
\usepackage{url}
\usepackage{natbib,bibentry}
\usepackage{amsfonts,amssymb,amsmath,bm}
\usepackage{graphics}
\graphicspath{{./Graphics/}}

\usepackage{pgf,pgfplots,pgfplotstable,tikz,tkz-fct}
\usepackage{hyperref}
\hypersetup{
  colorlinks,
  citecolor=blue,
  filecolor=blue,
  linkcolor=blue,
  urlcolor=blue
}

\usepackage{bibunits}

\DeclareMathOperator{\SD}{SD}
\DeclareMathOperator{\SE}{SE}
\DeclareMathOperator{\Var}{Var}
\DeclareMathOperator{\Cov}{Cov}

\DeclareMathOperator{\MEAN}{E}
\DeclareMathOperator{\bias}{bias}

\newcommand{\R}{\mathfrak{R}}
\newcommand{\RMS}{{\mathrm{RMS}}}
\newcommand{\MSE}{{\mathrm{MSE}}}
\newcommand{\RISK}{{\mathrm{R}}}
\newcommand{\Err}{{\mathrm{Err}}}
\newcommand{\TPF}{{\mathrm{TPF}}}
\newcommand{\FNF}{{\mathrm{FNF}}}
\newcommand{\FPF}{{\mathrm{FPF}}}
\newcommand{\AUC}{{\mathrm{AUC}}}

\DeclareSymbolFont{letters}{OML}{cmm}{m}{it}

\makeindex             

\usepackage[left=0.65in, right=0.65in, top=0.7in, bottom=0.65in]{geometry}

\begin{document}

\newcommand{\AUTHOR}{Waleed A. Yousef}%
\newcommand{\INSTITUTE}{Waleed A. Yousef\at ECE Dep. ISOT Laboratory, University of Victoria~\email{wyousef@uvic.ca}; and\\CS Dep.,
  HCILAB, Helwan University,~\email{wyousef@fci.helwan.edu}}

\frontmatter

\tableofcontents

\mainmatter

\begin{bibunit}[s-spbasic]
  \title*{Machine Learning Assessment: implications to cybersecurity}%
\author{\AUTHOR}%
\institute{\INSTITUTE}%
\maketitle%
\label{ch:MLassessment}

\abstract{%
  After discussing the construction of machine learning (ML) algorithms in the previous chapter, this chapter is dedicated to their
  assessment and performance estimation, with an emphasis on classification assessment, a topic that is equally important specially in
  the context of cyberphysical security design. The literature is full of nonparametric methods to estimate a statistic from just one
  available dataset through resampling techniques, e.g., jackknife, bootstrap and cross validation (CV). Special statistics of great
  interest are the error rate and the area under the ROC curve (AUC) of a classification rule. The importance of these resampling
  methods stems from the fact that they require no knowledge about the probability distribution of the data or the construction details
  of the ML algorithm. This chapter provides a concise review of this literature to establish a coherent theoretical framework for
  these methods that can estimate both the error rate (a one-sample statistic) and the AUC (a two-sample statistic). The resampling
  methods are usually computationally expensive, because they rely on repeating the training and testing of a ML algorithm after each
  resampling iteration. Therefore, the practical applicability of some of these methods may be limited to the traditional ML algorithms
  rather than the very computationally demanding approaches of the recent deep neural networks (DNN). In the field of cyberphysical
  security, many applications generate structured (tabular) data, which can be fed to all traditional ML approaches. This is in
  contrast to the DNN approaches, which favor unstructured data, e.g., images, text, voice, etc.; hence, the relevance of this chapter
  to this field.%
}


  \section{Introduction}\label{subsec:mylabel2}
\subsection{Motivation}\label{sec:motivation}
Consider a ML problem, where some models have been trained on a given dataset. It is then required to know their performances, in terms
of any performance measure, on the population of testers. This is not only for the sake of assessing each of them, but also to be able
to select the best model among them. These different models could even represent different instances of the same ML algorithm, with
different values of parameters (e.g., a KNN with different values of $K$), and it is required to choose the best value for the current
problem. The performance on the population of testers is called the true performance, because this is the performance on the whole
population, not on a subset of it.

\bigskip

If the underlying probability distribution of the testers is known, e.g., from a priori knowledge about the nature of the problem, the
true performance can be calculated mathematically. One of the first attempts in this direction was~\cite{Fukunaga1990Introduction}, where
he assumed the data follows a multinormal distribution, to find a closed-form expression of the error rate of a binary
classification rule. An alternative to mathematical calculations is simulating a very large dataset, from the assumed distribution,
from which a very accurate estimation of the true performance can be obtained.

The early work of Fukunuga was inspiring, from the theoretical point of view, for the early community of pattern recognition and
machine learning to understand important theoretical properties and concepts. However, for real-life applications it is very unusual
that the assumption of multinormality, or any other assumption, hold. In these situations, which are called nonparametric, or
distribution-free, it is impossible to derive either the true performance in closed form, or estimate it using a very large simulated
dataset. In such situations, the true performance must be estimated from a single testing dataset (testers). The way we obtain such a
testing dataset defines two major paradigms, discussed next.

\bigskip

In Paradigm I, we only have one dataset $\mathbf{t}$, usually called the design or construction dataset, from which we have to make up
a training dataset $\mathbf{tr}$ and a testing dataset $\mathbf{ts}$, such that $\mathbf{t}=\mathbf{tr}\cup\mathbf{ts}$. Otherwise,
training and testing on the same dataset $\mathbf{t}$ would provide a very optimistic estimate of the performance measure. This
splitting is performed iteratively using one of the resampling techniques, e.g., jackknife, bootstrap, or cross validation. In each
resampling iteration we get a different pair of training and testing datasets, on which the algorithm will be trained and tested,
respectively. The results from these different iterations will be compiled together, as defined by the resampling method, to provide a
single estimate of the performance measure. It is obvious that the performance estimation obtained from any of these methods will vary
with varying the design dataset $\mathbf{t}$. This chapter is dedicated to reviewing this paradigm, its different estimators, and the
variance estimation of these estimators.

It is worth mentioning that fatal fallacies are committed by practitioners when using this paradigm. For example, a
very common mistake is using the whole dataset $\mathbf{t}$ to learn some statistical properties of the different classes of the
classification problem, mistakenly naming this a data preprocessing step, using these properties to construct a classifier, then
excluding this step from the resampling mechanism afterwards. Although the correct way of performing preprocessing is explained in
textbooks~\citep[see, e.g.,][Sec. 7.10.2]{Hastie2009ElemStat}, we still see this mistake in several occasions in both academia and
industry.

\bigskip

In Paradigm II, it is required, or even mandated (e.g., in several public-policy-making or regulatory settings), to maintain what might
be called the traditional data hygiene of two independent datasets: the design dataset ${\rm {\bf t}}$, and a final testing dataset
${\rm {\bf TS}}$, which is a sequestered testing dataset that has never been available to the design procedure, but for just onetime
final testing. Assessing a ML algorithm from independent testing dataset is as simple as applying the estimators of the performance
measure of interest (Sec.~\ref{sec:notation}) on the testing dataset. However, the estimator will then have two sources of
variability, the design and the testing datasets. The mathematical details of this paradigm and the estimation of this variance are
discussed in~\cite{Yousef2006AssessClass,Chen2012ClassVar}, and not reviewed in our present chapter.

Although it may seem very safe to use this testing paradigm, some practitioners abuse it as well. One possible common mistake is that
they test several models on this sequestered testing set, then they analyze the relative estimated performances. Accordingly, these
models are redesigned to improve their performance on the testing set! Worse than this is keeping iterating this processes several
times, which indeed turns the independent sequestered testing dataset to be part of the training dataset, indirectly through this human
mental parsing of the results, which acts as a feedback that guides the redesign process.

\bigskip

Nowadays, it is almost the default in the field of ML to leverage both paradigms in the task of model assessment and selection. The
available dataset is initially split into two datasets:%
\begin{enumerate}
\item the design dataset $\mathbf{t}$, from which the ML algorithm is designed. This is conducted via one of the resampling methods of
  paradigm I explained above. Usually, several algorithms are used, and several parameters' values are examined for each algorithm.
  Then, the model with the best performance is chosen.

\item the sequestered testing dataset $\mathbf{TS}$, on which the final chosen model from paradigm I is assessed once and only once,
  without redesign. This is the final estimation of the performance measure that should be reported, along with the estimation of its
  variance.%
\end{enumerate}
It is worth mentioning that, there is a convention in the field to call the dataset $\mathbf{ts}$ that is split from the design dataset
$\mathbf{t}$ during the resampling process, a validation dataset rather than a testing dataset, to reserve the word testing to the
final testing datset $\mathbf{TS}$ of paradigm II. However, in some applications, the converse is adopted; i.e., $\mathbf{ts}$ is
called the a testing dataset and $\mathbf{TS}$ is called a validation dataset. To avoid ambiguity, any notation and expression should
be defined clearly within any context.

\bigskip

What is introduced above is valid for any ML problem, whether it is regression or classification, and for any performance measure,
whether it is the error rate $\Err$, AUC, or any other. However, we emphasize below two very important issues.

(1) The true performance, which we discussed its estimation in this introduction so far, is itself a random variable whose randomness
arises from the randomness of the training dataset, as was explained in the previous chapter. Have we changed the training
dataset, the true performance would change. For example, and without loss of generality (WLOG) but for the sake of illustration,
suppose the whole design dataset $\mathbf{t}$ is used as a training dataset $\mathbf{tr}$ and we are interested in the AUC as a
performance measure. Then, as was explained in the previous chapter, we should be interested in the following:
\begin{enumerate}
\item $\AUC_{{\rm {\bf t}}} $: the true performance conditional on a particular training dataset ${\rm {\bf t}}$ of a specified size $n$.
  
\item $\MEAN_{{\rm {\bf t}}} \AUC_{{\rm {\bf t}}} $: the expectation of true performance over the population of training datasets of the
  same size $n$.
  
\item $\Var_{\rm {\bf t}}\AUC_{{\rm {\bf t}}}$: the variance of the true performance over the population of training datasets of the
  same size $n$.
\end{enumerate}

(2) Regarding the meaning and utility of the performance measure, we emphasize the importance of the ROC curve and its AUC as a summary
measure~\citep{Hanley1982TheMeaning,Hanley1989ROCMeth,Bradley1997TheUseOfTheAreaUnder}, where the former is a manifestation of the
trade-off between the two types of error of any binary classification rule. We always advocate for the use of the ROC or its AUC since
they are prevalence independent; i.e., they do not depend on a particular chosen threshold, class prior probability, or
misclassification costs. Adopting a performance measure that is prevalence dependent, e.g., the overall accuracy or its many different
versions, can provide a misleading measure of the classification power of the classification algorithm, especially in classification
problems that involve, for instance, unbalanced data (different class size). Therefore, the present chapter assumes familiarity with
the ROC and its AUC, at the level provided in the previous chapter. However, for the sake of completeness, all notations are tersely
summarized in the following subsection.

\subsection{Notation}\label{sec:notation}
Consider the binary classification problem, where a classification rule $\eta$ gives a score of $h(x)$ for the predictor $x$, and
classifies it to one of the two classes by comparing this score $h(x)$ to a chosen threshold $th$. The observation $x$ belongs to one
of the two classes with distributions $F_i$, $i=1,2$. The two error components of this rule ($e_1$, or the false negative fraction
(FNF), and $e_2$ or the false positive fraction (FPF)), along with the risk, are given as follows:%
\begin{subequations}\label{Eqe1e2}
  \begin{align}
    \FNF &= e_{1}  =\int_{-\infty}^{th}{f_{h}}\left(  {h(x)|\omega_{1}}\right){dh(x)},\\
    \FPF &= e_{2}  =\int_{th}^{\infty}{f_{h}}\left(  {h(x)|\omega_{2}}\right)  {dh(x)},\\
    \RISK&=c_{12}P_{1}e_{1}+c_{21}P_{2}e_{2}.\label{eq44}
  \end{align}
\end{subequations}
The cost $c_{ij},\ i,j=1,2$ is the cost of classifying an observation as belonging to class $j$ whereas it belongs to class $i$;
$c_{ii}= 0$, which means there is no cost for correct classification; and $P_i$ is the prior probability of each class, $i=1,2$. The
risk~\eqref{eq44} is called the ``error rate'' $\Err$, or probability of misclassification (PMC), when putting $c_{12}=c_{21}=1$, which
is denoted by the 0-1 cost, or loss.

The receiver operating characteristics (ROC) curve is a plot of the true positive fraction (TPF), which is $1-\FNF$, versus the FPF.
Then the area under the curve (AUC) is given by:%
\begin{subequations}\label{eq47eq:PopulationManWhit}
  \begin{align}
    \AUC&=\int_{0}^{1}{\TPF~d(\FPF)}.\label{eq47}\\
        &=\Pr\bigl[ h(x)|\omega_{2} < h(x)|\omega_{1}\bigr],\label{eq:PopulationManWhit}
  \end{align}
\end{subequations}
which expresses how the classifier scores for class $\omega_1$ are stochastically larger than those of class $\omega_2$.

\bigskip

If the distributions $F_1$ and $F_2$ are not known, a setup that is called nonparametric or distribution-free, any performance measure
can be estimated only numerically from a given dataset, called the testing dataset. This is regardless of the testing paradigm, i.e.,
whether this testing dataset is obtained by simulation, resampling, or sequestering. This is done by assigning equal probability mass
for each observation:%
\begin{equation}
  \hat{F}:\text{mass}~\frac{1}{n}~\text{on}\,~t_{i},~i=1,\ldots,n,\label{eq52}
\end{equation}
where $n$ is the size of the testing dataset. Lemma~\ref{lem:perf-estim-from} shows that this is the maximum likelihood estimator (MLE)
of the distribution $F$.

In this case the performance measures~\eqref{Eqe1e2} can be obtained as follows.%
\begin{subequations}\label{eq55}
  \begin{align}
    \widehat{\FNF} &=\widehat{e_{1}} = \frac{1}{n}\sum\limits_{i=1}^{n}I_{h(x_{i}|\omega_{1})<th}\\
    \widehat{\FPF} &= \widehat{e_{2}} = \frac{1}{n}\sum\limits_{i=1}^{n}I_{h(x_{i}|\omega_{2})>th}\\
    \widehat{\RISK(\eta)}  & =\frac{1}{n}\left(  {c_{12}\,\widehat{e_{1}}\,n_{1}+c_{21}\,\widehat{e_{2}}\,n_{2}}\right).
  \end{align}
\end{subequations}
The indicator function $I_{cond}$ equals 1 or 0 when the Boolean expression $cond$ is true or false, respectively. The values $n_{1}$
and $n_{2}$ are the number of observations in the two classes respectively, and $\widehat{P_{1}}$ and $\widehat{P_{2}}$ are the
estimated a priori probabilities for each class.

\bigskip

As the the two components $\TPF$ and $\FPF$ defined a single operating point on the ROC, the two components $\widehat{\TPF}
(=1-\widehat{\FNF})$ and $\widehat{\FPF}$ give one point on the empirical (estimated) ROC curve. To draw the complete curve in the
nonparametric situation, the classifier's sore is calculated for each point of the available dataset. Then all possible thresholds are
considered in turn, i.e., the threshold values between every two successive scores. At each threshold value a point on the ROC curve is
calculated. Then the AUC~\eqref{eq47} can be estimated from the empirical ROC curve using the trapezoidal rule:%
\begin{equation}
  \widehat{\AUC}=\frac{1}{2}\sum\limits_{i=2}^{n_{th}}{\left(  {\FNF_{i}-\FNF_{i-1}}\right)  (\TPF_{i}+\TPF_{i-1})},\label{eq57-01}%
\end{equation}
where $n_{th}$ is the number of threshold values taken over the dataset. By plotting the empirical ROC curve, it is easy to see
that~\eqref{eq57-01} is the same as the Mann-Whitney statistic---which is another form of the Wilcoxon rank-sum test
\citep[Ch.4]{Hajek1999TheoryOfRank}---defined by:%
\begin{subequations}\label{eq58}
  \begin{gather}
    \widehat{\AUC}=\frac{1}{n_{1}n_{2}}\sum\limits_{j=1}^{n_{2}}{\sum\limits_{i=1}^{n_{1}}{\psi\left(  {h\left(  {x_{i}|\omega_{1}}\right),h\left(  {x_{j}|\omega_{2}}\right)  }\right)  }},\\
    \psi(a,b)=\left\{%
      \begin{array}[c]{ccc}
        1 &  & a>b\\
        1/2 &  & a=b\\
        0 &  & a<b
      \end{array}\right..
  \end{gather}
\end{subequations}
It is interesting, as well, to know from the theory of $U$-statistics~\citep{Randles1979IntroductionTo} that the estimator~\eqref{eq58}
is the uniform minimum variance unbiased estimator (UMVUE) for the probability~\eqref{eq:PopulationManWhit} under the
distribution~\eqref{eq52}.

\bigskip

All the estimators given above have the nice property of converging to their corresponding population definitions,~\eqref{Eqe1e2}
and~\eqref{eq47eq:PopulationManWhit}, as the size of the testing set goes to infinity. It is worth mentioning that each of the error
estimators $\hat{e}_1$ and $\hat{e}_2$ in~\eqref{eq55} is called a one-sample statistic, because its kernel $I_{(\cdot)}$ requires only one
observation from either distributions. However, the AUC estimator in~\eqref{eq58} is a two-sample statistic since its kernel
$\psi(\cdot,\cdot)$ requires two observations, one from each distribution. This is a fundamental difference between both estimators
(statistics) which will be treated and explained carefully in the present chapter.

\subsection{Roadmap}\label{sec:manuscript-roadmap}
The rest of this chapter is organized as follows. Sec.~\ref{sec:nonp-meth-bias} paves the road to the chapter by reviewing the
nonparametric estimators for estimating the mean and variance of one-sample statistics, including the preliminaries of bootstraps and
influence function. This section is a very concise review mainly of the work done in
\cite{Hampel1974TheInfluence},~\cite{Efron1993AnIntroduction}, and~\cite{Huber1996RobustStatistical}. Sec.~\ref{subsec:estimating}
switches gears and reviews the nonparametric estimators that estimate the mean and variance of a special kind of statistics, i.e., the
error rate of classification rules. This section is a concise review of the work done mainly in~\cite{Efron1983EstimatingTheError}
and~\cite{Efron1997ImprovementsOnCross}. Sec.~\ref{sec:NonParamInfAUC} explains how the nonparametric estimators that estimate the
error rate, a one-sample statistic, can be extended to estimate the AUC, a two-sample statistic. It does so by providing theoretical
parallelism between the two sets of estimators and showing that the extension is rigorous and not just an ad hoc application.
Sec.~\ref{sec:conclusion} concludes the chapter and provides a discussion and an advice for practitioners.


  \section{Nonparametric Methods for Estimating the Bias and the Variance of a Statistic}\label{sec:nonp-meth-bias}
Consider a statistic $s$ that is a function of a dataset $\mathbf{x}:\{x_{i},\,i=1,\ldots,n\}$, where $x_{i}\overset{i.i.d}{\sim}F$.
The statistic $s$ is now a random variable and its variability comes from the variability of $x_{i}$. Suppose that this statistic is
used to estimate a real-valued parameter $\theta=f\left( F\right) $. Then $\hat{\theta }=s\left( \mathbf{x}\right) $ has expected value
$\MEAN {s\left( \mathbf{x}\right) }$ and variance $\Var s\left( \mathbf{x}\right)$. The mean
squared error (MSE) of the estimator $\hat{\theta}$ is defined as:%
\begin{equation}
  \MSE(\hat{\theta})=\MEAN\left[  {\hat{\theta}-\theta}\right]  ^{2}.\label{eq59}%
\end{equation}
The root of the mean squared error (RMS) has the same units and is on the same scale of the original variable $\theta$, and hence has more
intuitive value. The bias of the estimator $\hat{\theta}=s\left( \mathbf{x}\right) $ is defined by the difference between the true
value of the parameter and the expectation of the estimator, i.e.,%
\begin{equation}
  \bias_{F}\left(\hat{\theta}\right)  =\MEAN_{F}  s\left(\mathbf{x}\right)  - \theta.\label{eq60}
\end{equation}
Then, it is known that, the MSE in (\ref{eq59}) can be decomposed to:%
\begin{equation}
  \MSE(\hat{\theta})=\bias_{F}^{2}\left(\hat{\theta}\right) +\Var_F \hat{\theta}.\label{eq61}%
\end{equation}
A critical question is whether the bias and variance of the statistic $s$ in (\ref{eq61}) may be estimated from the available dataset?

\subsection{Bootstrap Estimate}\label{subsubsec:bootstrap}
\begin{figure}[t]\centering
  \includegraphics[width=\textwidth]{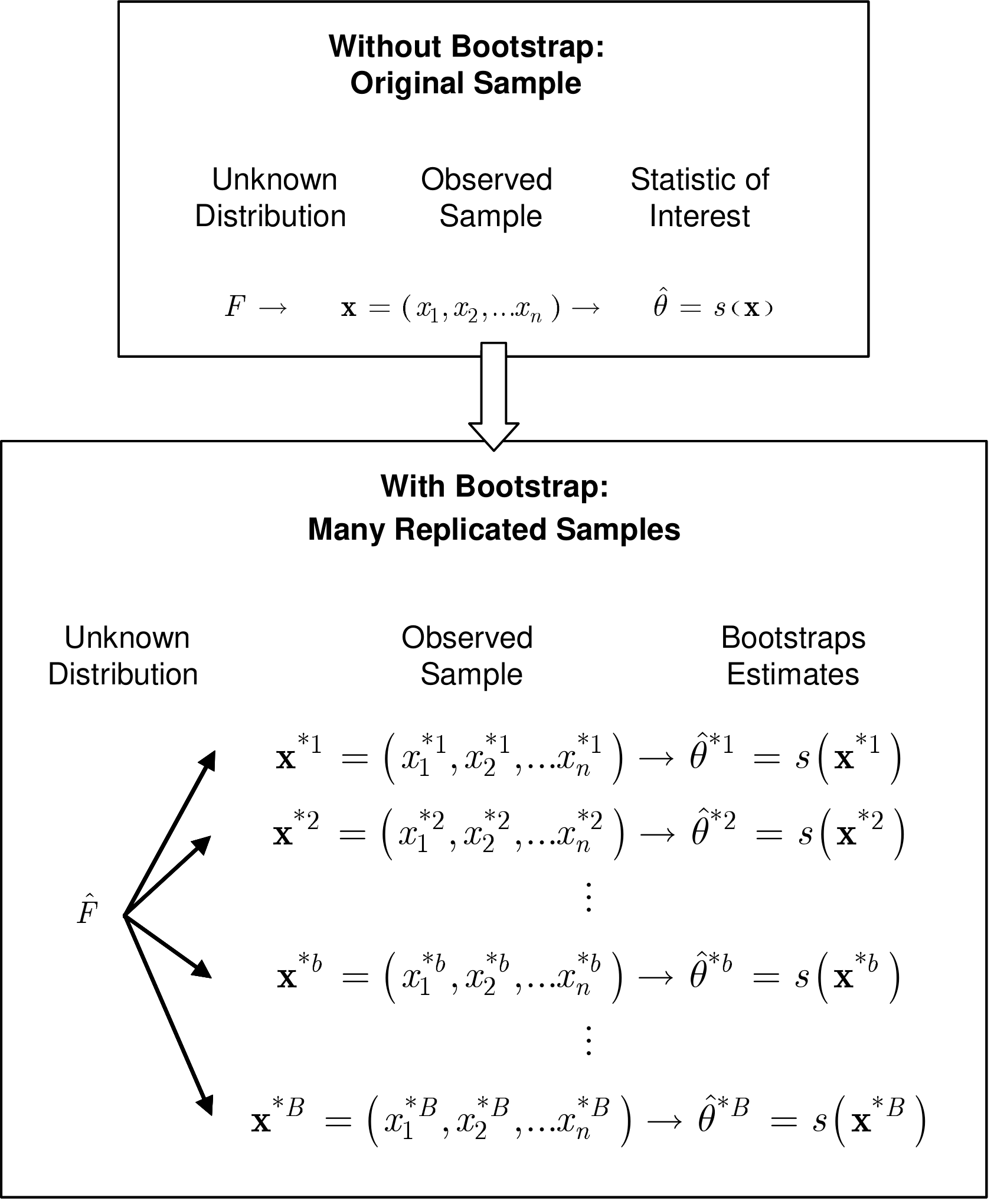}%
  \caption{Bootstrap mechanism: $B$ bootstrap replicates are withdrawn (by sampling and replacement) from the original sample. From
    each replicate the statistic is calculated. (The idea behind this figure first appeared in~\cite[Figure 6.1, pp.
    48]{Efron1993AnIntroduction})}\label{fig5}
\end{figure}
The bootstrap was introduced by~\cite{Efron1979BootstrapMethods} to estimate the standard error of a statistic. The bootstrap mechanism
is implemented by treating the current dataset $\mathbf{x}$ as a representation for the population distribution $F$; i.e.,
approximating the distribution $F$ by the MLE defined in (\ref{eq52}). Then $B$ bootstrap samples are drawn from that empirical
distribution. Each bootstrap replicate is of size $n$, the same size as $\mathbf{x}$, and is obtained by sampling with replacement.
Then in a bootstrap replicate some case $x_{i}$, in general, will appear more than once at the expense of another $x_{j}$ that will not
appear. The original dataset will be treated now as the population, and the replicates will be treated as samples from the population.
This situation is illustrated in~\figurename~\ref{fig5}. Each of these bootstrap replicates is denoted by $\mathbf{x}^{\ast b},\
b=1,\ldots,B$, and the corresponding bootstrap replications of the statistics $\hat{\theta}=s(\mathbf{x})$ itself are given by:%
\begin{equation}
  \hat{\theta}^{\ast b}=s(\mathbf{x}^{\ast b}),\quad b=1,\ldots,B,
\end{equation}
The bootstrap estimate of bias and standard error are defined by:%
\begin{gather}
  \bias_{B}(\hat{\theta})=\hat{\theta}^{\ast}-\hat{\theta},\label{eq62}\\
  \widehat{\SE}_B=\left[ \frac{1}{(B-1)} {\sum\limits_{b=1}^{B}{\left[  {\hat{\theta}^{\ast b}-\hat{\theta}^{\ast}}\right]  ^{2}}}\right]
  ^{1/2},\label{eq63}\\
  \hat{\theta}^{\ast}=\frac{1}{B}\sum\limits_{b=1}^{B}{\hat{\theta}^{\ast b}}.
\end{gather}
Either in estimating the bias or the standard error, the larger the number of bootstraps $B$ the closer the estimate to the asymptotic
value, i.e.,%
\begin{equation}
  \lim_{B\rightarrow\infty}\widehat{\SE}_B(\hat{\theta}^{\ast})=\SE_{\hat{F}}(\hat{\theta}^{\ast}).\label{eq64}
\end{equation}
For more details and some examples the reader is referred to~\citet[Ch. 6, 7, and 10]{Efron1993AnIntroduction}.

\subsection{Jackknife Estimate}\label{subsubsec:jackknife}
Instead of replicating from the original dataset, a new set $\mathbf{x} _{(i)}$ is created
by removing the case $x_{i}$ from the dataset. Then the jackknife samples are defined by:%
\begin{equation}
  \mathbf{x}_{\left(  i\right)  }=(x_{1},\ldots,x_{i-1},x_{i+1},\ldots ,x_{n}),\quad i=1,\ldots,n,\label{eq65}
\end{equation}
and the $n$-jackknife replications of the statistic $\hat{\theta}$ are:%
\begin{equation}
  \hat{\theta}_{(i)}=s(\mathbf{x}_{\left(  i\right)  }),\quad i=1,\ldots,n.\label{eq66}
\end{equation}
The jackknife estimates of bias and standard error are defined by:%
\begin{gather}
  \widehat{\bias}_{J}=(n-1)(\hat{\theta}^J-\hat{\theta}\label{eq67}),\\
  \widehat{\SE}_J=\left[  {\frac{n-1}{n}\sum\limits_{i=1}^{n}{(\hat{\theta}_{\left(  i\right)  }-\hat{\theta}^J)^{2}}}\right]  ^{1/2},\\
  \hat{\theta}^J=\frac{1}{n}\sum\limits_{i=1}^{n}{\hat{\theta}_{\left(i\right) }}.
\end{gather}
For motivation behind the factors $(n-1)$ and $(n-1)/n$ in (\ref{eq67}) see~\cite[Ch. 11]{Efron1993AnIntroduction}. The jackknife
estimate of variance is discussed in detail in~\cite{Efron1981NonparametricEstimates} and~\cite{Efron1981TheJacknifeEstimate}.

\subsection{Bootstrap vs. Jackknife}\label{subsubsec:mylabel3}
Usually, it requires up to 200 bootstraps to yield acceptable bootstrap estimates; (in special situations like estimating the
uncertainty in classifier performance it may take up to thousands of bootstraps). Hence, this requires calculating the statistic
$\hat{\theta}$ the same number of times $B$, as well. In the case of the jackknife, it requires only $n$ calculations as shown in (\ref{eq66}).
If the sample size is smaller than the required number of bootstraps, the jackknife is more economical in terms of computational cost.

In terms of accuracy, the jackknife can be seen to be an approximation to the bootstrap when estimating the standard error of a
statistic~\citep[Ch. 20]{Efron1993AnIntroduction}. Thus, if the statistic is linear they almost give the same result; (the bootstrap
gives the jackknife estimate multiplied by $[(n-1)/n]^{1/2}$). A statistic $s(\mathbf{x})$ is said to be linear if:%
\begin{equation}
  s(\mathbf{x})=\mu+\frac{1}{n}\sum\limits_{i=1}^{n}{\alpha(x_{i})},
  \label{eq68}%
\end{equation}
where $\mu$ is a constant and $\alpha(\cdot)$ is a function. This also can be viewed as having one data point at a time in the argument
of the function $\alpha$. Similarly, the jackknife can be seen as an approximation to the bootstrap when estimating the bias. If the
statistic is quadratic, they almost agree except in a normalizing factor . A statistic $s(\mathbf{x})$ is quadratic if:%
\begin{equation}
  s(\mathbf{x})=\mu+\frac{1}{n}\sum\limits_{1\leq i\leq n}{\alpha(x_{i})+\frac{1}{n^{2}}\sum\limits_{1\leq i<j\leq n}{\beta(x_{i},x_{j})}}.\label{eq69}
\end{equation}
An in-depth treatment of the bootstrap and jackknife, and their relation to each other, in mathematical detail is provided by
\citet[Ch. 1-5]{Efron1982TheJackknife}.

If the statistic is not smooth the jackknife will fail. Informally speaking, a statistic is said to be smooth if a small change in the
data leads to a small change in the statistic. An example of a non-smooth statistic is the median. If the sample cases are ranked and
the median is calculated, it will not change when a sample case changes unless this sample case bypasses the median value. Using the
same argument, we can see that an example of a smooth statistic is the sample mean.

\subsection{Influence Function, Infinitesimal Jackknife, and Estimate of Variance}\label{subsubsec:influence}
The infinitesimal jackknife was introduced by~\cite{Jaeckel1972TheInfinitesimal}. The concept of the influence curve was introduced
later by~\cite{Hampel1974TheInfluence}. In the present context and for pedagogical purposes, the influence curve will be explained
before the infinitesimal jackknife, since the former can be understood as the basis for the latter.

Following~\cite{Hampel1974TheInfluence}, let $\R$ be the real line and $s$ be a real-valued functional defined on the
distribution $F$, which is defined on $\R$. The distribution $F$ can be perturbed by adding some probability measure (mass) on
a point $x$. This should be balanced by a decrement in $F$ elsewhere, resulting in a new probability distribution $G_{\varepsilon,x}$
defined by:%
\begin{equation}
  G_{\varepsilon,x}=(1-\varepsilon)F+\varepsilon\delta_{x},~x\in\R.\label{eq70}
\end{equation}
Then, the influence curve $IC_{s,F}(\cdot)$ is defined by:%
\begin{equation}
  IC_{s,F}(x)=\lim_{\varepsilon\rightarrow0^{+}}\frac{s\left(  \left(1-\varepsilon\right)  F+\varepsilon\delta_{x}\right)  -s\left(  F\right)}{\varepsilon}.\label{eq71}
\end{equation}
It should be noted that $F$ does not have to be a discrete distribution. A simple example of applying the influence curve concept is to
consider the expectation $s=\int{x~dF(x)}=\mu$. Substituting back in (\ref{eq71}) gives:%
\begin{equation}
  IC_{s,F}(x)=x-\mu.\label{eq72}
\end{equation}
The meaning of this formula is the following: the rate of change of the functional $s$ with the probability measure at a point $x$ is
$x-\mu$. This is how the point $x$ influences the functional $s$. The influence curve can be used to linearly approximate a functional
$s$, along with its variance, which is similar to taking up to only the first-order term in a Taylor series expansion
(Appendix~\ref{sec:influence-function}).

\bigskip

It is important to state here that $s$ should be a functional in $\hat{F}$ that is an approximation to $F$, as was initially assumed in
(\ref{eq71}). If for example the value of the statistic $s$ changes if every sample case $x_{i}$ is duplicated, i.e., repeated twice,
this is not a functional statistic. An example of a functional statistic is the biased version of the variance estimate
$\Sigma_{i}(x_{i}-\bar{x}_{i})^{2}/n$, while the unbiased version $\Sigma_{i}(x_{i}-\bar{x}_{i})^{2}/(n-1)$ is not a functional
statistic. Generally, any approximation $s(\hat{F})$ to the functional $s(F)$, by approximating $F$ by the MLE $\hat{F}$, obviously
will be functional. In such a case the statistic $s(\hat{F})$ is called the plug-in estimate of the functional $s(F)$. Moreover, the
influence function (IF) method for variance estimation is applicable only to those functional statistics whose derivative (\ref{eq76})
exists. If that derivative exists, the statistic is called a smooth statistic; i.e., a small change in the dataset leads a small change
in the statistic. For instance, although the median is a functional statistic in the sense that duplicating any sample case will result
in the same value of the median, it is not smooth as described at the end of Sec.~\ref{subsubsec:mylabel3}. A key reference for the
IF is~\cite{Hampel1986RobustStatistics}. Appendix~\ref{sec:influence-function} shows an interesting connection to the
jackknife estimate.


  \section{Nonparametric Methods for Estimating the Error Rate of a Classification Rule}\label{subsec:estimating}
The review provided in this section is a terse summary of the main work of
Efron~\citep{Efron1997ImprovementsOnCross,Efron1993AnIntroduction,Efron1983EstimatingTheError}. In the previous section the statistic,
or generally speaking the functional, was a function of just one dataset. For a non-fixed design, i.e., when the predictors of the
testing set do not have to be the same as the predictors of the training dataset, a slight clarification for the previous notations is
needed. The classification rule trained on the training dataset $\mathbf{t}$ will be denoted as $\eta_{\mathbf{t}}$. Any new
observation that does not belong to $\mathbf{t}$ will be denoted by $t_{0}=(x_0,y_0)$. Therefore, the classification loss is
given by $L(y_{0},\eta_{\mathbf{t}}(x_{0}))$. Any performance measure conditional on that training dataset will be similarly
subscripted. Thus, all the performance measures should be subscripted $\mathbf{t}$; and hence the risk and the error
rate~\eqref{Eqe1e2} should be denoted by $\RISK_{\mathbf{t}}$ and $\Err_{\mathbf{t}}$, respectively. In the sequel, for simplicity and
WLOG, the $0$-$1$ loss function will be used. In such a case the conditional error rate will be given by:%
\begin{equation}
  \Err_{\mathbf{t}}=\MEAN_{0F}L\left(  y_{0},\eta_{\mathbf{t}}\left(x_{0}\right)  \right),\quad \left(  x_{0},y_{0}\right)  \sim F.\label{eq83}
\end{equation}
The expectation $\MEAN_{0F}$ is subscripted so to emphasize that it is taken over the observations $t_{0}\notin\mathbf{t}$. If the
performance is measured in terms of the error rate and we are interested in the mean performance, not the conditional one, then it is
given by:%
\begin{equation}
  \Err=\MEAN_{\mathbf{t}}\Err_{\mathbf{t}}.\label{eq:MeanError}
\end{equation}
The expectation $\MEAN_{\mathbf{t}}$ is the expectation over the training dataset $\mathbf{t}$, which would be the same if we had written $\MEAN_{F}$; for
notation clarity the former is chosen.

Consider a classification rule $\eta_{\mathbf{t}}$ already trained on a training dataset $\mathbf{t}$. A natural next question is, given
that there is just a single dataset available, how to use this dataset in assessing the classifier performance as well? Said
differently, how should one estimate, using only the available dataset, the true classification performance of a classification rule in
predicting new observations; these observations are different from those on which the rule was trained. In this section, we will review
the principal methods in the literature for estimating both the true error rate~\eqref{eq83} and its mean~\eqref{eq:MeanError}, of a
classification rule.

\subsection{Apparent Error}\label{subsubsec:apparent}
The apparent error is the error of the fitted model when it is tested on the same training data. Of course it is downward biased with
respect to the true error rate since it results from testing on the same information used in training~\citep{Efron1986HowBiasedIs}. The
apparent error is defined by:%
\begin{subequations}
  \begin{align}
    \overline{\Err}_{\mathbf{t}}  &  =\MEAN_{\hat{F}}L(y,\eta_{\mathbf{t}}(x)),\quad (x,y)\in\mathbf{t}\label{eq82}\\
                                 &  =\frac{1}{n}\sum\limits_{i=1}^{n}\left[I_{\hat{h}_{\mathbf{t}}(x_{i}|\omega_{1})<th}+I_{\hat{h}_{\mathbf{t}}(x_{i}|\omega_{2})>th}\right].
  \end{align}
\end{subequations}
Overfitting a classifier to minimize the apparent error is not the goal. The goal is to minimize the true error rate (\ref{eq83}) or
its mean~\eqref{eq:MeanError}.

\subsection{Cross Validation (CV)}\label{subsubsec:cross}
The basic concept of CV, as a resampling approach, has been proposed in different articles since the mid-1930s. The concept simply
leans on splitting the data into two parts; the first part is used in design (or training) without any involvement of the second part.
Then the second part is used to test the designed procedure; this is to test how the designed procedure will behave for new datasets.
\cite{Stone1974CrossValidatory} is a key reference for CV that proposes different criteria for optimization.

CV can be used to assess the prediction error of a model or in model selection. The true error rate in (\ref{eq83}) is the expected
error rate for a classification rule if tested on the population, conditional on a particular training dataset $\mathbf{t}$. This
performance measure can be approximated by the leave-one-out CV (LOOCV) by:%
\begin{equation}
  \widehat{\Err}_{\mathbf{t}}^{_{cv1}}=\frac{1}{n}\sum\limits_{i=1}^{n}{L}\left({y_{i},\eta_{\mathbf{t}^{(i)}}(x_{i})}\right)  ,~~~(x_{i},y_{i})\in\mathbf{t}.\label{eq84}
\end{equation}
This is done by training the classification rule on the dataset $\mathbf{t}^{\left( i\right) }$ that does not include the case $t_{i}$;
then testing the trained rule on that omitted case. This proceeds in \textquotedblleft round-robin\textquotedblright\ fashion until all
cases have contributed one at a time to the error rate. There is a hidden assumption in this mechanism: the training dataset
$\mathbf{t}$ will not change very much by omitting a single case. Therefore, testing on the omitted observation one at a time accounts
for testing approximately the same trained rule on $n$ new cases, all different from each other and different from those the classifier
has been trained on. Besides this LOOCV, there are other versions named $K$-fold (or leave-$n/K$-out). In such versions the whole
dataset is split into $K$ roughly equal-sized subsets, each of which contains approximately $n/K$ observations. The classifier is
trained on $K-1$ subsets and tested on the left-out one; hence we have $K$ iterations. It is clear that the LOOCV is a special case of
the $K$-fold CV, where $K=n$.

It is of interest to assess this estimator to see whether it estimates the conditional true error
$\MEAN\left[ {\widehat{\Err}_{\mathbf{t}}^{_{cv1}}-\Err_{\mathbf{t}}}\right] ^{2}$, with small MSE, as was designed or not. Many
simulation results, e.g., \cite{Efron1983EstimatingTheError}, show that there is only a very weak correlation between the CV estimator
$\widehat{\Err}_{\mathbf{t}}^{_{cv1}}$ and the conditional true error rate $\Err_{\mathbf{t}}$. This issue is discussed in mathematical
detail in the excellent paper by~\cite{Zhang1995AssessingPrediction}. Those other estimators that are based on resampling as well, and
will be reviewed below, are shown to have this same attribute. This very interesting (and perhaps surprising) result means the
following: whether the estimator is designed to estimate the conditional performance or the mean performance it indeed estimates the
latter because of the weak correlation with the former.

\subsection{Bootstrap Methods for Error Rate Estimation}\label{subsubsec:mylabel4}
The prediction error in (\ref{eq83}) is a function of the training dataset $\mathbf{t}$ and the testing population $F$. Bootstrap
estimation can be implemented here by treating the empirical distribution $\hat{F}$ as an approximation to the actual population
distribution $F$. By replicating from that distribution one can simulate many training datasets $\mathbf{t}^{*b},~b=1,\ldots,B$. For
every replicated training dataset the classifier will be trained and then tested on the original dataset $\mathbf{t}$. This is the
simple bootstrap (SB) estimator approach~\citep[Sec. 17.6]{Efron1993AnIntroduction} that was defined formally by:%
\begin{equation}
  \widehat{\Err}_{\mathbf{t}}^{_{SB}}=\MEAN_{\ast}\sum\limits_{i=1}^{n}{L(y_{i},\eta_{\mathbf{t}^{\ast}}(x_{i}))/n,\quad \hat{F}\rightarrow\mathbf{t}^{\ast}}.\label{eq85}
\end{equation}
It should be noted that this estimator no longer estimates the true error rate (\ref{eq83}) because the expectation taken over the
bootstraps mimics an expectation taken over the population of trainers, i.e., it is not conditional on a particular training dataset.
Rather, the estimator (\ref{eq85}) estimates the expected performance of the classifier $\MEAN_{F}\Err_{\mathbf{t}}$. For a finite
number of bootstraps, the expectation (\ref{eq85}) can be approximated by:
\begin{equation}
  \widehat{\Err}_{\mathbf{t}}^{_{SB}}=\frac{1}{B}\sum\limits_{b=1}^{B}{\sum\limits_{i=1}^{n}{L}}\left(  {{y_{i},\eta_{\mathbf{t}^{\ast b}}(x_{i})}}\right) /n.\label{eq86}
\end{equation}

\subsubsection{Leave-One-Out Bootstrap (LOOB)}\label{para:leave}
The previous estimator is obviously biased since the original dataset $\mathbf{t}$ used for testing includes part of the training data
in every bootstrap replicate.~\cite{Efron1983EstimatingTheError} proposed that, after training the classifier on every bootstrap
replicate, it is tested on those cases in the set $\mathbf{t}$ that are not included in the training; this concept can be developed as
follows. Eq. (\ref{eq86}) can be rewritten by interchanging the order of the double summation to give:%
\begin{equation}
  \widehat{\Err}_{\mathbf{t}}^{_{SB}}=\frac{1}{n}\sum\limits_{i=1}^{n}{\sum\limits_{b=1}^{B}{L}}\left(  {{y_{i},\eta_{\mathbf{t}^{\ast b}}(x_{i})}}\right)\left/B\right..\label{eq87}
\end{equation}
This equation is formally identical to (\ref{eq86}) but it expresses a different mechanism for evaluating the same quantity. It says
that, for a given point, the average performance over the bootstrap replicates is calculated; then this performance is averaged over
all the $n$ cases. Now, if every case $t_{i}$ is tested only from those bootstraps that did not include it in the training, a slight
modification of the previous expression yields the leave-one-out bootstrap (LOOB) estimator:%
\begin{equation}
  \widehat{\Err}_{\mathbf{t}}^{_{\left(  1\right)  }}=\frac{1}{n}\sum\limits_{i=1}^{n}\left[  {\sum\limits_{b=1}^{B}{I_{i}^{b}L}}\left({{y_{i},\eta_{\mathbf{t}^{\ast b}}(x_{i})}}\right)  \left/ \sum\limits_{{b}^{\prime}=1}^{B}{I_{i}^{{b}^{\prime}}}\right. \right],\label{eq88}
\end{equation}
where the indicator function $I_{i}^{b}$ equals one when the case $t_{i}$ is not included in the training replicate $b$, and zero
otherwise. \cite{Efron1997ImprovementsOnCross} emphasized a critical point about the difference between this bootstrap estimator and
the LOOCV. The CV tests on a given sample case $t_{i}$, having been trained just once on the remaining dataset. By contrast, the LOOB
tests on a given sample case $t_{i}$ using a large number of classifiers that result from a large number of bootstrap replicates that
do not contain that sample. This results in a smoothed cross-validation-like estimator. We explained and elaborated on this smoothness
property in~\cite{Yousef2019LeisurelyLookVersionsVariants-arxiv}.

\subsubsection{The Refined Bootstrap (RB)}\label{para:mylabel2}
The SB and the LOOB, from their definitions, look like designed to estimate the mean true error rate~\eqref{eq:MeanError} of a
classifier. For estimating the true conditional error rate of a classifier, conditional on a particular training dataset,
\cite{Efron1983EstimatingTheError} proposed to correct for the downward biased estimator $\overline{\Err}_{\mathbf{t}}$. Since the true
error rate $\Err_{\mathbf{t}}$ can be written as $\overline{\Err}_{\mathbf{t}%
}+(\Err_{\mathbf{t}}-\overline{\Err}_{\mathbf{t}})$, then it can be approximated by
$\overline{\Err}_{\mathbf{t}}+\MEAN_{F}(\Err_{\mathbf{t}}-\overline {\Err}_{\mathbf{t}})$. The term
$(\Err_{\mathbf{t}}-\overline{\Err}_{\mathbf{t}})$ is called the optimism. The expectation of the optimism can be approximated over the
bootstrap population. Finally the refined bootstrap approach, as named in~\citet[Sec. 17.6]{Efron1993AnIntroduction}, gives the
estimator:%
\begin{equation}
  \widehat{\Err}_{\mathbf{t}}^{_{RB}}=\overline{\Err}_{\mathbf{t}}+\MEAN_{\ast}u\Err_{\mathbf{t}\ast}(\hat{F})-\overline{\Err}_{\mathbf{t}\ast}u,\label{eq89}
\end{equation}
where $\Err_{\mathbf{t}\ast}(\hat{F})$ represents the error rate obtained from training the classifier on all bootstrap replicates
$\mathbf{t}^{\ast}$ and testing on the empirical distribution $\hat{F}$. This can be approximated for a limited number of bootstraps
by:%
\begin{equation}
  \widehat{\Err}_{\mathbf{t}}^{_{RB}}=\overline{\Err}_{\mathbf{t}}+\frac{1}{B}\sum\limits_{b=1}^{B}\left[  {\sum\limits_{i=1}^{n}{L}}\left(  {{y_{i},\eta_{\mathbf{t}^{\ast b}}(x_{i})}}\right)  {/n-\sum\limits_{i=1}^{n}{L}}\left(  {{y_{ib}^{\ast},\eta_{\mathbf{t}^{\ast b}}(x_{ib}^{\ast})}}\right){/n}\right].\label{eq90}
\end{equation}

\subsubsection{The 0.632 Bootstrap}\label{para:mylabel3}
If the concept used in developing the LOOB estimator, i.e., testing on cases not included in training, is used again in estimating the
optimism described above, this gives the 0.632 bootstrap estimator. Since the probability of including a case $t_{i}$ in the bootstrap
$\mathbf{t}^{\ast b}$ is given by:%
\begin{align}
  \Pr(t_{i}  \in\mathbf{t}^{\ast b}) &=1-(1-1/n)^{n} \approx1-e^{-1}=0.632,\label{eq91}
\end{align}
the effective number of sample cases contributing to a bootstrap replicate is approximately 0.632 of the size of the training dataset.
\cite{Efron1983EstimatingTheError} introduced the concept of a \textit{distance} between a point and a sample in terms of a
probability. Having trained on a bootstrap replicate, testing on those cases in the original dataset not included in the bootstrap
replicate accounts for testing on a set far from the training one, i.e., the bootstrap replicate. This is because every sample case in
the testing set has zero probability of belonging to the training dataset, i.e., very distant from the training dataset. This is a reason for
why the LOOB is an upwardly biased estimator.~\cite{Efron1983EstimatingTheError} showed roughly that:%
\begin{equation}
  \MEAN_{F}\left[  {\Err_{\mathbf{t}}-\overline{\Err}_{\mathbf{t}}}\right] \approx0.632\,\MEAN_{F}\left[  {\widehat{\Err}_{\mathbf{t}}^{_{\left(  1\right) }}-\overline{\Err}_{\mathbf{t}}}\right].\label{eq92}
\end{equation}
Substituting back in (\ref{eq89}) gives the 0.632 estimator:%
\begin{equation}
  \widehat{\Err}_{\mathbf{t}}^{_{(.632)}}=.368\,\overline{\Err}_{\mathbf{t} }+.632\,\widehat{\Err}_{\mathbf{t}}^{_{\left(  1\right)  }}.\label{eq93}
\end{equation}
The proof of the above results can be found in~\citet{Efron1983EstimatingTheError} and~\citet[Sec. 6]{Efron1993AnIntroduction}.

The motivation behind this estimator as stated earlier is to correct for the downward biased apparent error by adding a piece of the
upward biased LOOB estimator. But an increase in variance should be expected as a result of adding this piece of the relatively
variable apparent error. Moreover, this new estimator is no longer smooth since the apparent error itself is unsmooth.

\subsubsection{The 0.632+ Bootstrap Estimator}\label{para:mylabel4}
The .632 estimator reduces the bias of the apparent error. But for over-trained classifiers, i.e., those whose apparent error tends to
be zero, the .632 estimator is still downward biased.~\cite{Breiman1984ClassificationAnd} provided the example of an overfitted rule,
like $1$NN where the apparent error is zero. If, however, the class labels are assigned randomly to the predictors the
true error rate will obviously be 0.5. But substituting in (\ref{eq93}) gives an estimate of $.632\times.5=.316$. To account for
this bias for such over-fitted classifiers,~\cite{Efron1997ImprovementsOnCross} defined the \textit{no-information error rate} $\gamma$
by:%
\begin{equation}
  \gamma=\MEAN_{0F_{ind}}  {L}\left(  {y_{0},\eta_{\mathbf{t}}(x_{0})}\right),\label{eq94}
\end{equation}
where $F_{ind}$ means that $x_{0}$ and $y_{0}$ are distributed marginally as $F$ but they are independent. Or said differently, the
label is assigned randomly to the predictor. Then for a training sample $\mathbf{t}$, $\gamma$ can be estimated by:%
\begin{equation}
  \hat{\gamma}= \frac{1}{{n^{2}}}\sum\limits_{i=1}^{n}{\sum\limits_{j=1}^{n}{L}}\left({{y_{i},\eta_{\mathbf{t}}(x_{j})}}\right).\label{eq95}
\end{equation}
This means that the $n$ predictors have been permuted with the $n$ responses to produce $n^{2}$ non-informative cases. In the special
case of binary classification, let $\hat{p}_{1}$ be the proportion of the response classified as belonging to class 1. Also, let
$\hat{q}_{1}$ be the proportion of the responses classified as belonging to class 1. Then (\ref{eq95}) reduces to:%
\begin{equation}
  \hat{\gamma}=\hat{p}_{1}(1-\hat{q}_{1})+(1-\hat{p}_{1})\hat{q}_{1}.\label{eq96}
\end{equation}
Also define the \textit{relative overfitting rate}:%
\begin{equation}
  \hat{R}=\frac{\widehat{\Err}_{\mathbf{t}}^{_{\left(  1\right)  }}-\overline{\Err}_{\mathbf{t}}}{\hat{\gamma}-\overline{\Err}_{\mathbf{t}}}.\label{eq97}
\end{equation}
\cite{Efron1997ImprovementsOnCross} showed that the bias of the .632 estimator for the case of over-fitted classifiers is alleviated by
using a renormalized version of that estimator:%
\begin{subequations}
  \begin{align}
    \widehat{\Err}_{\mathbf{t}}^{_{(.632+)}}  &  =(1-\hat{w})\overline{\Err}_{\mathbf{t}}+\hat{w}\widehat{\Err}_{\mathbf{t}}^{_{\left(  1\right)  }},\label{eq98}\\
    \hat{w}  &  =\frac{.632}{1-.368\hat{R}}.
  \end{align}
\end{subequations}
It is useful to express the .632+ estimator in terms of its predecessor, the .632 estimator. Combining (\ref{eq93}), (\ref{eq96}), and
(\ref{eq97}) then substituting in (\ref{eq98}) yields:%
\begin{equation}
  \widehat{\Err}_{\mathbf{t}}^{_{\left(  {.632+}\right)  }}=\widehat{\Err}_{\mathbf{t}}^{_{\left(  {.632}\right)  }}+(\widehat{\Err}_{\mathbf{t}}^{_{\left(  1\right)  }}-\overline{\Err}_{\mathbf{t}})\frac{.368\cdot.632\cdot\hat{R}}{1-.368\hat{R}}.\label{eq99}
\end{equation}
\cite{Efron1997ImprovementsOnCross} consider the possibility that $\hat{R}$ lies out of the region $\left[ {0,1}\right] $. This leads
to their proposal of defining:%
\begin{align}
  \widehat{\Err}_{\mathbf{t}}^{_{\left(  1\right)  }\prime}  &  =\min(\widehat{\Err}_{\mathbf{t}}^{_{\left(  1\right)  }},\hat{\gamma}),\label{eq100}\\
  {\hat{R}}^{\prime}  &  =\left\{  {%
                        \begin{array}[c]{lll}
                          (\widehat{\Err}_{\mathbf{t}}^{_{\left(  1\right)  }}-\overline{\Err}_{\mathbf{t}})/(\hat{\gamma}-\overline{\Err}_{\mathbf{t}}) &  & \overline{\Err}_{\mathbf{t}}<\widehat{\Err}_{\mathbf{t}}^{_{\left(  1\right)  }}<\gamma\\
                          0 &  & \text{otherwise}
                        \end{array}}\right.,
\end{align}
to obtain a modification to (\ref{eq99}) that finally becomes:%
\begin{equation}
  \widehat{\Err}_{\mathbf{t}}^{\left(  {.632+}\right)  }=\widehat{\Err}_{\mathbf{t}}^{\left(  {.632}\right)  }+(\widehat{\Err}_{\mathbf{t}}^{_{\left(1\right)  }\prime}-\overline{\Err}_{\mathbf{t}})\frac{.368\cdot.632\cdot{\hat{R}}^{\prime}}{1-.368{\hat{R}}^{\prime}}.\label{eq101}
\end{equation}

\subsection{Estimating the Standard Error of Error Rate Estimators}\label{subsec:mylabel3}
What have been reviewed above are several resampling methods: the CV, .632, and .632+ estimate the conditional error rate of a
classification rule, conditional on that training dataset; and the LOOB estimates the mean error rate, where the expectation is taken
over the population of training datasets. Regardless of what the estimator is designed to estimate, it is still a function of the
current dataset $\mathbf{t}$, i.e., it is a random variable. If, e.g., the LOOB estimator $\widehat{\Err}_{\mathbf{t}}^{_{\left(
      1\right) }}$ is considered, it estimates a constant real-valued parameter $\MEAN_{0F}\MEAN_{F}L(y_{0},\eta_{\mathbf{t}}(x_{0}))$
with expectation taken over all the trainers and then over all the testers, respectively; this is the overall mean error rate. Yet,
$\widehat{\Err}_{\mathbf{t}}^{_{\left( 1\right) }}$ is a random variable whose variability comes from the finite size of the available
dataset. If the classifier is trained and tested on a very large number of observations, this would approximate training and testing on
the entire population, and the variability would shrink to zero. This also applies for any performance measure other than the error
rate. So, we are interested now in estimating $\Var_{\mathbf{t}}\widehat{\Err}_{\mathbf{t}}^{_{\left(1\right) }}$, the variance of the
estimator, not estimating $\Var_{\mathbf{t}}\Err_{\mathbf{t}}$, the variance of the true performance.

\bigskip

The next question then is, having estimated the mean performance of a classifier: what is the associated uncertainty of this estimate.
Said differently: can an estimate of the variance of this estimator be obtained from the same training
dataset?~\cite{Efron1997ImprovementsOnCross} proposed the use of the IF method (Sec.~\ref{subsubsec:influence}), to estimate the
uncertainty (variability) in $\widehat{\Err}_{\mathbf{t}}^{_{\left( 1\right) }}$. The reader is alerted that estimators that
incorporate a piece of the apparent error are not suitable for the IF method. Such estimators are not smooth because the apparent error
itself is not smooth.

By recalling the definitions of Sec.~\ref{subsubsec:influence}, $\widehat{\Err}_{\mathbf{t}}^{_{\left( 1\right) }}$ is now the statistic
$s(\hat{F})$. To simplify notation, the error $L(y_{i},\eta_{\mathbf{t}^{\ast b}}(x_{i}))$ may be denoted by $L_{i}^{b}$, and define
the following notation:%
\begin{equation}
  l_{\cdot}^{b}=\frac{1}{n}\sum\limits_{i=1}^{n}{I_{i}^{b}L_{i}^{b}},\label{eq102}%
\end{equation}
Also, define $N_{i}^{b}$ to be the number of times the case $t_{i}$ is included in the bootstrap $b$. Then, it has been proven
in~\cite{Efron1995CrossValidation} that the IF of such an estimator is given by:%
\begin{equation}
  \left.  {\frac{\partial s(\hat{F}_{\varepsilon,i})}{\partial\varepsilon} }\right\vert _{\varepsilon=0}=(2+\frac{1}{n-1})(\hat{E}_{i}-\widehat {\Err}_{\mathbf{t}}^{_{\left(  1\right)  }})+\frac{n\sum\nolimits_{b=1} ^{B}{(N_{i}^{b}-\bar{N}_{i})I_{i}^{b}}}{\sum\nolimits_{b=1}^{B}{I_{i}^{b}}}.\label{eq103}
\end{equation}
Combining (\ref{eq81}) and (\ref{eq103}) gives an estimation to the uncertainty in $\widehat{\Err}_{\mathbf{t}}^{_{\left( 1\right) }}$.


  \section{Nonparametric Methods for Estimating the AUC of a Classification Rule}\label{sec:NonParamInfAUC}
In the present section, we extend the study carried out in~\cite{Efron1983EstimatingTheError,Efron1997ImprovementsOnCross}, and
summarized in Sec.~\ref{subsec:estimating}, to construct nonparametric estimators for the AUC (a two-sample statistic) analogue to
those of the error rate (a one-sample statistic). Although some previous experimental comparative
studies~\citep{Yousef2004ComparisonOf,Sahiner2001ResamplingSchemes,Sahiner2008ClassifierPerformance} were conducted to compare some of
these resampling-based AUC estimators, in particular the .632 versions, there was no theoretical justification of using these
estimators for the AUC. We provide here a full account of the different versions of bootstrap estimators reviewed in
Sec.~\ref{subsec:estimating} and show how they can be formally extended to estimate the AUC.

\subsection{Construction of Nonparametric Estimators for AUC}\label{sec:constr-nonp-estim}
Before switching to the AUC, some more elaboration on Sec.~\ref{subsec:estimating} is needed. The SB estimator~\eqref{eq85} can be
rewritten as:%
\begin{equation}
  \widehat{\Err}_{\mathbf{t}}^{_{SB}}=\MEAN_{\ast}\MEAN_{\widehat{F}}\left[  {L(\eta_{\mathbf{t}^{\ast} }(x),y)|\mathbf{t}^{\ast}} \right].
\end{equation}
Since there would be some observation overlap between $\mathbf{t\ }$ and $\mathbf{t}^{\ast}$, this approach suffers an obvious bias as
was introduced in that section. This was the motivation behind interchanging the expectations and defining the LOOB
(Sec.~\ref{para:leave}). Alternatively, we could have left the order of the expectation but with testing on only those observations in
$\mathbf{t}$ that do not appear in the bootstrap replication $\mathbf{t}^{\ast}$, i.e., the distribution $\hat {F}^{\left( \ast\right)
}$. The parenthesis notation $\left( \ast\right) $ refers to excluding from $\widehat{F}$, in the testing stage, the training cases
$\mathbf{t}^{*}$ that were generated from the bootstrap replication. We call the resulting estimator
$\widehat{\Err}_{\mathbf{t}}^{_{\left( \ast\right) }}$, which we define formally by:%
\begin{align}
  \widehat{\Err}_{\mathbf{t}}^{_{\left(  \ast\right)  }}  &  =\MEAN_{\ast}\MEAN_{\hat{F}^{\left( \ast\right)  }}\left[  {L(\eta_{\mathbf{t}^{\ast}}(x),y)|\mathbf{t}^{\ast} }\right]\label{Eq14PRL}
\end{align}
We can give the inner expectation the notation $\Err_{\mathbf{t}^{\ast b}}(\widehat{F}^{\left( \ast\right)  })$, and rewrite the estimator as:
\begin{subequations}\label{eq:Errstar}
  \begin{align}
    \widehat{\Err}_{\mathbf{t}}^{_{\left(  \ast\right)  }} &= \MEAN_{\ast}{\Err_{\mathbf{t}^{\ast b}}(\widehat{F}^{\left( \ast\right)  })}\\
                                             &  =\frac{1}{B}\sum\limits_{b=1}^{B}\left[  \sum\limits_{i=1}^{N}{I_{i}^{b}L(\eta_{\mathbf{t}^{\ast b}}(x_{i}),y_{i}) \left/ \sum\limits_{i'=1}^{N}{I_{i'}^b}\right.}\right],
  \end{align}
\end{subequations}
where the indicator $I_{i}^{b}$ equals one if the observation $t_{i}$ is excluded from the bootstrap replication $\mathbf{t}^{\ast b}$, and
equals zero otherwise. The inner expectation in (\ref{Eq14PRL}) is taken over those observations not included in the bootstrap
replication $\mathbf{t}^{\ast}$, whereas the outer expectation is taken over all the bootstrap replications.

\bigskip

Analogously to Sec.~\ref{subsec:estimating}, and to what has been introduced above, we can define several bootstrap estimators for
the AUC. The start is the SB estimate, which can be defined as:%
\begin{subequations}
  \begin{align}
    \widehat{\AUC}_{\mathbf{t}}^{_{SB}}  &  =\MEAN_{\ast}{\AUC_{\mathbf{t}^{\ast}}(\widehat{F})},\quad \widehat{F}\rightarrow\mathbf{t}^{\ast} \label{eq111}\\
                                      &  =\MEAN_{\ast}\left[  {\frac{1}{n_{1}n_{2}}\sum\limits_{j=1}^{n_{2}}{\sum\limits_{i=1}^{n_{1}}{\psi(\hat{h}_{\mathbf{t}^{\ast}}(x_{i}),\hat{h}_{\mathbf{t}^{\ast}}(x_{j}))}}}\right],\quad x_{i}\in\omega_{1},\ x_{j}\in\omega_{2}.
  \end{align}
\end{subequations}
This averages the Mann-Whitney statistic over the bootstraps, where $\AUC_{\mathbf{t}^{\ast}}(\widehat{F})$ refers to the AUC obtained
from training the classifier on the bootstrap replicate $\mathbf{t}^{\ast}$ and testing it on the empirical distribution $\widehat{F}$.
In the approach used here, the bootstrap replicate $\mathbf{t}^{\ast}$ preserves the ratio between $n_{1}$ and $n_{2}$, which is called
stratification. That is, the training sample $\mathbf{t}$ is treated as
$\mathbf{t}=\mathbf{t}_{1}\cup\mathbf{t}_{2},\ \mathbf{t}_{1}\in\omega_{1},\ \mathbf{t}_{2}\in\omega_{2}$; then $n_{1}$ cases are
replicated from the first-class sample and $n_{2}$ cases are replicated from the second-class sample to produce $\mathbf{t}_{1}^{\ast}$
and $\mathbf{t}_{2}^{\ast}$ respectively, where $\mathbf{t}^{\ast}=\mathbf{t}_{1}^{\ast}\cup\mathbf{t}_{2}^{\ast}$. This was not needed
when the performance measure was the error rate since it is a statistic that does not operate simultaneously on two different sets of
observations as the Mann-Whitney statistic does (in $U$-statistic theory~\citep{Randles1979IntroductionTo}, error rate and Mann-Whitney
are called one-sample and two-sample statistics respectively). The expectation (\ref{eq111}) is approximated by averaging over a finite
number of bootstrap:%
\begin{equation}
  \widehat{\AUC}_{\mathbf{t}}^{_{SB}}=\frac{1}{B}\sum\limits_{b=1}^{B}{\AUC_{\mathbf{t}^{\ast b}}(\widehat{F})},\label{eq112}
\end{equation}

\bigskip

The same motivation behind the estimator (\ref{eq88}) can be applied here, i.e., testing only on those cases in $\mathbf{t}$ that are
not included in the training dataset $\mathbf{t}^{\ast b}$, in order to reduce the bias. This can be carried out in (\ref{eq112}) without
interchanging the summation order. The new estimator is named $\widehat{\AUC}_{\mathbf{t}}^{_{(\ast)}}$, where the parenthesis notation
$\left( \ast\right) $ refers to the exclusion, in the testing stage, of the training cases that were generated from the bootstrap
replication. Formally, we define this as:%
\begin{subequations}\label{eq:AUCstar}
  \begin{align}
    \widehat{\AUC}_{\mathbf{t}}^{_{\left(  \ast\right)  }}  &  =\MEAN_{\ast}{\AUC_{\mathbf{t}^{\ast b}}(\widehat{F}^{\left( \ast\right)  })}\\
                                                         &  =\frac{1}{B}\sum\limits_{b=1}^{B}\left[  {\sum\limits_{j=1}^{n_{2}} {\sum\limits_{i=1}^{n_{1}}{\psi({{{\hat{h}_{\mathbf{t}^{\ast}}(x_{i}),\hat {h}_{\mathbf{t}^{\ast}}(x_{j}))}}}I_{i}^{b}I_{j}^{b}}}} \left/ \sum_{i^{\prime} =1}^{n_{1}}I_{i^{\prime}}^{b}\sum_{j^{\prime}=1}^{n_{2}}I_{j^{\prime}} ^{b} \right. \right].
  \end{align}
\end{subequations}
\bigskip

The RB and 0.632 estimators can be introduced here in the same way it was used for the true error rate (Sec.~\ref{para:mylabel3}) as:%
\begin{equation}
  \widehat{\AUC}_{\mathbf{t}}^{RB}=\overline{\AUC}_{\mathbf{t}}+\MEAN_{\ast} \left[\AUC_{\mathbf{t}\ast}(\widehat{F})-\overline{\AUC}_{\mathbf{t}\ast}\right].\label{eq114}
\end{equation}
Then, if testing is carried out on cases excluded from the bootstraps, analogously to the 0.632 estimator of the error rate, this gives
rise to the 0.632 estimator of the AUC:%
\begin{equation}
  \widehat{\AUC}_{\mathbf{t}}^{_{(.632)}}=.368\,\overline{\AUC}_{\mathbf{t} }+.632\,\widehat{\AUC}_{\mathbf{t}}^{_{\left(  \ast\right)  }}.\label{eq115}
\end{equation}
It should be noted that this estimator is designed to estimate the true AUC for a classifier trained on the dataset $\mathbf{t}$ (the
classifier performance conditional on the training dataset $\mathbf{t})$. This is on contrary to the estimator (\ref{eq:AUCstar}) that
estimates the mean performance of the classifier (this is the expectation over the training dataset population for the conditional
performance).

\bigskip

The 0.632+ estimator $\widehat{\AUC}_{\mathbf{t}}^{_{(.632+)}}$ develops from $\widehat{\AUC}_{\mathbf{t}}^{_{(.632)}}$ in the same way as
$\widehat {\Err}_{\mathbf{t}}^{_{(.632+)}}$ developed from $\widehat{\Err}_{\mathbf{t}}^{_{(.632)}}$ in Sec.~\ref{para:mylabel4}.
There are two modifications to the details. The first regards the \textit{no-information error rate} $\gamma$; it can be proven that
the \textit{no-information} AUC is given by $\gamma_{\AUC} = 0.5$ (Lemma~\ref{lem:NoInfo}). The second regards the definitions
(\ref{eq100}), which should be modified to accommodate for the AUC. The new definitions are given by:%
\begin{subequations}
  \begin{align}
    \widehat{\AUC}_{\mathbf{t}}^{_{\left(  {.632+}\right)  }} &=\widehat{\AUC}_{\mathbf{t}}^{_{\left(  {.632}\right)  }}+(\widehat{\AUC}_{\mathbf{t}}^{_{\left(  \ast\right)  }\prime}-\overline{\AUC}_{\mathbf{t}})\frac{.368\cdot.632\cdot{\hat{R}}^{\prime}}{1-.368{\hat{R}}^{\prime}},\label{eq18IEEE}\\
    \widehat{\AUC}_{\mathbf{t}}^{_{\left(  \ast\right)  }\prime} &=\max\bigl( \widehat{\AUC}_{\mathbf{t}}^{_{\left(  \ast\right)  }},\gamma_{\AUC}\bigr),\label{eq19IEEE}\\
    {\hat{R}}^{\prime}&=\left\{%
                       \begin{array}[c]{ccc}%
                         \frac{(\widehat{\AUC}_{\mathbf{t}}^{_{\left(  \ast\right)  }}-\overline{\AUC}_{\mathbf{t}})}{(\gamma_{\AUC}-\overline{\AUC}_{\mathbf{t}})} &  & \text{if}~\overline{\AUC}_{\mathbf{t}}>\widehat{\AUC}_{\mathbf{t}}^{\left(\ast\right)  }>\gamma_{\AUC}\\
                         0 &  & \text{otherwise}
                       \end{array}\right..
  \end{align}
\end{subequations}

\bigskip

To this end, we have constructed the AUC nonparametric estimators analogue to those of the error rate. Some of them, mainly the .632+
estimator, will have the least bias~\citep{Efron1997ImprovementsOnCross}. However, all of these estimators are not ``smooth'' and not
eligible for the variance estimation via, e.g., the IF method (Sec.s~\ref{subsubsec:influence} and \ref{subsec:mylabel3}). The only
estimator that may seem smooth, is the star versions $\widehat{\Err}_{\mathbf{t}}^{_{\left( \ast\right) }}$ and
$\widehat{\AUC}_{\mathbf{t}}^{_{\left( \ast\right) }}$. However, the inner components
${\Err_{\mathbf{t}^{\ast b}}(\widehat{F}^{\left( \ast\right) })}$ and ${\AUC_{\mathbf{t}^{\ast b}}(\widehat{F}^{\left( \ast\right) })}$
are unsmooth themselves, because the classifier is trained on just one dataset. Applying the influence function enforces distributing
the differential operator $\partial/\partial\varepsilon$, of the IF, over the summation to be encountered by these unsmooth components.

\subsection{The Leave-Pair-Out Boostrap (LPOB) $\widehat{\AUC}^{_{\left( {1,1}\right) }}$, Its
  Smoothness and Variance Estimation}\label{sec:need-leave-pair}
The above discussion suggests introducing an analogue to $\widehat {\Err}_{\mathbf{t}}^{_{\left( 1\right) }}$ for measuring the performance
in AUC. This estimator is motivated from (\ref{eq111}) the same way the estimator $\widehat{\Err}_{\mathbf{t}}^{_{\left( 1\right) }}$
was motivated from (\ref{eq87}). The SB estimator (\ref{eq111}) can be rewritten as:%
\begin{align}
  \widehat{\AUC}_{\mathbf{t}}^{SB}  &  =\frac{1}{n_{1}n_{2}}\sum\limits_{j=1}^{n_{2}}{\sum\limits_{i=1}^{n_{1}}{\MEAN_{\ast}  {\psi(\hat{h}_{\mathbf{t}^{\ast}}(x_{i}),\hat{h}_{\mathbf{t}^{\ast}}(x_{j}))}}}\label{eq124}\\
                                  &  =\frac{1}{n_{1}n_{2}}\sum\limits_{j=1}^{n_{2}}{\sum\limits_{i=1}^{n_{1}}{\sum\limits_{b=1}^{B}\left[  {\psi(\hat{h}_{\mathbf{t}^{\ast b}}(x_{i}),\hat{h}_{\mathbf{t}^{\ast b}}(x_{j})){\left/B\right.}}\right]}}.%
\end{align}
In words, the procedure is to select a pair (one observation from each class) and calculate for that pair the mean---over many
bootstrap replications and training---of the Mann-Whitney kernel. Then, average over all possible pairs. This procedure will be
optimistically biased because sometimes the testers will be the same as the trainers. To eliminate that bias, the inner bootstrap
expectation should be taken only over those bootstrap replications that do not include the pair $(t_{i},t_{j})$ in the training. Under
that constraint, the estimator (\ref{eq124}) becomes the leave-pair-out bootstrap (LPOB) estimator:%
\begin{subequations}\label{Eq18PRL}
  \begin{align}
    \widehat{\AUC}_{\mathbf{t}}^{_{\left(  {1,1}\right)  }} &=\frac{1}{n_{1}n_{2}}\sum \limits_{j=1}^{n_{2}}{\sum\limits_{i=1}^{n_{1}}{\widehat{\AUC}_{i,j}} },\\
    \widehat{\AUC}_{i,j}&=\sum\limits_{b=1}^{B}{I_{j}^{b}I_{i}^{b}\psi(\hat{h}_{\mathbf{t}^{\ast b}}(x_{i}),\hat{h}_{\mathbf{t}^{\ast b}}(x_{j})) \left/ \sum\limits_{{b}^{\prime}=1}^{B}{I_{j}^{{b}^{\prime}}I_{i}^{{b}^{\prime}}}\right. }.
  \end{align}
\end{subequations}
The two estimators $\widehat{\AUC}_{\mathbf{t}}^{_{\left( \ast\right) }}$ and $\widehat{\AUC}_{\mathbf{t}}^{_{\left( {1,1}\right) }}$ produce very similar results;
this is expected since they both estimate the same thing, i.e., the mean AUC. However, the inner component $\widehat{\AUC}_{i,j}$ of the
estimator $\widehat{\AUC}_{\mathbf{t}}^{_{\left( {1,1}\right) }}$ also enjoys the smoothness property of $\widehat{\Err}_{\mathbf{t}}^{_{\left( 1\right) }}$.

\subsection{Estimating the Standard Error of AUC Estimators}\label{sec_SDest}
The only smooth nonparametric estimator for the AUC so far is the LPOB estimator~\eqref{Eq18PRL}.~\cite{Yousef2005EstimatingThe}
discusses how to extend the approach of estimating the uncertainty in the error rate estimator using the IF method
(Sec.~\ref{subsec:mylabel3}) to estimate the uncertainty of this estimator, where interested readers may be referred to for all
mathematical details and experimental results that show that the IF method provides almost unbiased estimation for the
standard error of the LPOB estimator.


  \section{Illustrative Numerical Examples} \label{sec:illustr-exampl}
\subsection{Error Rate Estimation}\label{sec:comp-among-estim-1}
\begin{table}[t] \centering
  \begin{tabular}[c]{cc}\toprule
    \textbf{Estimator} & \textbf{Average RMS}\\\midrule
    \multicolumn{1}{l}{$\Err_{\mathbf{t}}$} & $0$\\
    \multicolumn{1}{l}{$\widehat{\Err}_{\mathbf{t}}^{_{\left(  1\right)  }}$} & $.083$\\
    \multicolumn{1}{l}{$\widehat{\Err}_{\mathbf{t}}^{_{(.632)}}$} & $.101$\\
    \multicolumn{1}{l}{$\widehat{\Err}_{\mathbf{t}}^{_{(.632+)}}$} & $.081$\\
    \multicolumn{1}{l}{$\overline{\Err}_{\mathbf{t}}$} & $.224$\\\bottomrule
  \end{tabular}
  \caption{Average of RMS error of each estimator over 24 experiments run by~\cite{Efron1997ImprovementsOnCross}. The estimator
    $\widehat{\Err}_{\mathbf{t}}^{\left( 1\right) }$ is the next to the estimator $\widehat{\Err}_{\mathbf{t}}^{\left( .632+\right) }$
    with only 2.5\% increase in RMS.}\label{tab2}
\end{table}
\cite{Efron1983EstimatingTheError} and~\cite{Efron1997ImprovementsOnCross} provide comparisons of their proposed estimators (discussed
in Section~\ref{subsec:estimating}). They ran many simulations considering a variety of classifiers and data distributions, as well as
real datasets. They assessed the estimators in terms of the RMS, the root of the experimental MSE:%
\begin{subequations}\label{eq104}
  \begin{align}
    \MSE  &  =\MEAN_{MC}(  {\widehat{\Err}_{\mathbf{t}}-\Err_{\mathbf{t}}})^{2}\\
          &  =\frac{1}{G}\sum\limits_{g=1}^{G}{(  {\widehat{\Err}_{\mathbf{t}_{g}}-\Err_{\mathbf{t}_{g}}})  ^{2}},
  \end{align}
\end{subequations}
where $\widehat{\Err}_{\mathbf{t}_{g}}$ is the estimator (any estimator) conditional on a training dataset $\mathbf{t}_{g}$, and
$\Err_{\mathbf{t}_{g}}$ is the true prediction error conditional on the same training dataset. The number of MC trials $G$ in their
experiments was 200. The following statement is quoted from~\cite{Efron1997ImprovementsOnCross}:
\begin{quotation}
  The results vary considerably from experiment to experiment, but in terms of RMS error the .632+ rule is an overall winner.
\end{quotation}
This conclusion was without stating the criterion for deciding the \textit{overall winner}. It was apparent from their results that the
$.632+$ rule is the winner in terms of the bias---as was designed for. We calculated the average of the RMS of every estimator across
all the 24 experiments they ran;~\tablename~\ref{tab2} displays these averages. The estimators $\widehat{\Err}_{\mathbf{t}}^{_{\left(
      1\right) }}$ and $\widehat{\Err}_{\mathbf{t}}^{_{(.632+)}}$ are quite comparable to each other with only 2.5\% increase in the
average RMS of the former. We will show below in Section~\ref{sec:comp-among-estim} that the AUC estimators exhibit the same behavior
but with magnified difference between the two estimators.

\subsection{AUC Estimation}\label{sec:comp-among-estim}
\begin{figure}[t]\centering
  \includegraphics[height=2.5in]{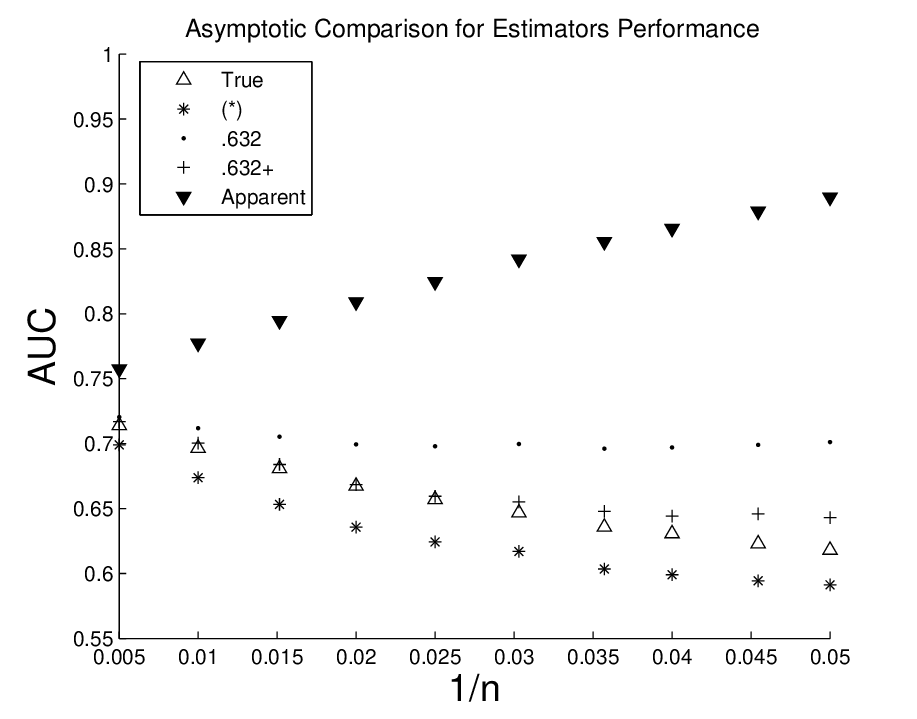}
  \caption{Comparison of the three bootstrap estimators, $\widehat{\AUC}_{\mathbf{t}}^{_{\left( \ast\right) }}$,
    $\widehat{\AUC}_{\mathbf{t}}^{_{\left( .632\right) }}$, and $\widehat {\AUC}_{\mathbf{t}}^{_{\left( .632+\right) }}$ for
    $5$-feature predictor. The $\widehat{\AUC}_{\mathbf{t}}^{_{\left( \ast\right) }}$ is downward biased, while the
    $\widehat{\AUC}_{\mathbf{t}}^{_{\left( .632\right) }}$ is an over correction for that bias. $\widehat{\AUC}_{\mathbf{t}}^{_{\left(
          .632+\right) }}$ is almost the unbiased version of the $\widehat{\AUC}_{\mathbf{t}}^{_{\left( .632\right) }}$. The figure
    first appeared in~\cite{Yousef2004ComparisonOf}.}\label{fig4IEEE}
\end{figure}

\begin{table}[t]\centering
  \resizebox{!}{3in}{    \begin{tabular}{ccccccc}
      \toprule
      \textbf{Estimator} & \textbf{Mean} & \textbf{SD} & \textbf{RMS} & \textbf{RMS$_{AM}$} & $\rho$ & \textbf{Size}\\
      \midrule
      \multicolumn{1}{l}{$\AUC_{\mathbf{t}}$} & 0.6181 & 0.0434 & 0 & 0.0434 & 1.0000 & \\
      \multicolumn{1}{l}{$\widehat{\AUC}_{\mathbf{t}}^{_{\left( \ast\right) }}$} & 0.5914 & 0.0947 & 0.0973 & 0.0984 & 0.2553 & \\
      \multicolumn{1}{l}{$\widehat{\AUC}_{\mathbf{t}}^{_{\left( .632\right) }}$} & 0.7012 & 0.0749 & 0.1128 & 0.1119 & 0.2559 & 20\\
      \multicolumn{1}{l}{$\widehat{\AUC}_{\mathbf{t}}^{_{\left( .632+\right)}}$} & 0.6431 & 0.0858 & 0.0906 & 0.0894 & 0.2218 & \\
      \multicolumn{1}{l}{$\overline{\AUC}_{\mathbf{t}}$} & 0.8897 & 0.0475 & 0.2774 & 0.2757 & 0.2231  & \\
      \midrule
      \multicolumn{1}{l}{$\AUC_{\mathbf{t}}$} & 0.6231 & 0.0410 & 0 & 0.0410 & 1.0000 & \\
      \multicolumn{1}{l}{$\widehat{\AUC}_{\mathbf{t}}^{_{\left( \ast\right) }}$} & 0.5945 & 0.0947 & 0.0956 & 0.0990 & 0.2993 & \\
      \multicolumn{1}{l}{$\widehat{\AUC}_{\mathbf{t}}^{_{\left( .632\right) }}$} & 0.6991 & 0.0763 & 0.1066 & 0.1077 & 0.3070 & 22\\
      \multicolumn{1}{l}{$\widehat{\AUC}_{\mathbf{t}}^{_{\left( .632+\right)}}$} & 0.6459 & 0.0846 & 0.0863 & 0.0876 & 0.2726 & \\
      \multicolumn{1}{l}{$\overline{\AUC}_{\mathbf{t}}$} & 0.8788 & 0.0499 & 0.2615 & 0.2606 & 0.2991& \\
      \midrule
      \multicolumn{1}{l}{$\AUC_{\mathbf{t}}$} & 0.6308 & 0.0400 & 0 & 0.0400 &1.0000 & \\
      \multicolumn{1}{l}{$\widehat{\AUC}_{\mathbf{t}}^{_{\left( \ast\right) }}$} & 0.5991 & 0.0865 & 0.0897 & 0.0922 & 0.2946 & \\
      \multicolumn{1}{l}{$\widehat{\AUC}_{\mathbf{t}}^{_{\left( .632\right) }}$} & 0.6971 & 0.0701 & 0.0961 & 0.0965 & 0.2997 & 25\\
      \multicolumn{1}{l}{$\widehat{\AUC}_{\mathbf{t}}^{_{\left( .632+\right)}}$} & 0.6442 & 0.0817 & 0.0815 & 0.0828 & 0.2758 & \\
      \multicolumn{1}{l}{$\overline{\AUC}_{\mathbf{t}}$} & 0.8656 & 0.0471 & 0.2406 & 0.2395 & 0.2833& \\
      \midrule
      \multicolumn{1}{l}{$\AUC_{\mathbf{t}}$} & 0.6359 & 0.0358 & 0 & 0.0358 &1.0000 & \\
      \multicolumn{1}{l}{$\widehat{\AUC}_{\mathbf{t}}^{_{\left( \ast\right) }}$} & 0.6035 & 0.0840 & 0.0874 & 0.0901 & 0.2904 & \\
      \multicolumn{1}{l}{$\widehat{\AUC}_{\mathbf{t}}^{_{\left( .632\right) }}$} & 0.6962 & 0.0688 & 0.0906 & 0.0915 & 0.2934 & 28\\
      \multicolumn{1}{l}{$\widehat{\AUC}_{\mathbf{t}}^{_{\left( .632+\right)}}$} & 0.6479 & 0.0792 & 0.0785 & 0.0802 & 0.2719 & \\
      \multicolumn{1}{l}{$\overline{\AUC}_{\mathbf{t}}$} & 0.8554 & 0.0472 & 0.2253 & 0.2246 & 0.2747& \\
      \midrule
      \multicolumn{1}{l}{$\AUC_{\mathbf{t}}$} & 0.6469 & 0.0343 & 0 & 0.0343 &1.0000 & \\
      \multicolumn{1}{l}{$\widehat{\AUC}_{\mathbf{t}}^{_{\left( \ast\right) }}$} & 0.6170 & 0.0750 & 0.0792 & 0.0807 & 0.2746 & \\
      \multicolumn{1}{l}{$\widehat{\AUC}_{\mathbf{t}}^{_{\left( .632\right) }}$} & 0.6997 & 0.0623 & 0.0818 & 0.0817 & 0.2722 & 33\\
      \multicolumn{1}{l}{$\widehat{\AUC}_{\mathbf{t}}^{_{\left( .632+\right)}}$} & 0.6553 & 0.0761 & 0.0752 & 0.0766 & 0.2656 & \\
      \multicolumn{1}{l}{$\overline{\AUC}_{\mathbf{t}}$} & 0.8419 & 0.0439 & 0.2010 & 0.1999 & 0.2434& \\
      \midrule
      \multicolumn{1}{l}{$\AUC_{\mathbf{t}}$} & 0.6571 & 0.0308 & 0 & 0.0308 &1.0000 & \\
      \multicolumn{1}{l}{$\widehat{\AUC}_{\mathbf{t}}^{_{\left( \ast\right) }}$} & 0.6244 & 0.0711 & 0.0753 & 0.0783 & 0.3185 & \\
      \multicolumn{1}{l}{$\widehat{\AUC}_{\mathbf{t}}^{_{\left( .632\right) }}$} & 0.6981 & 0.0598 & 0.0710 & 0.0725 & 0.3167 & 40\\
      \multicolumn{1}{l}{$\widehat{\AUC}_{\mathbf{t}}^{_{\left( .632+\right)}}$} & 0.6595 & 0.0739 & 0.0707 & 0.0739 & 0.3092 & \\
      \multicolumn{1}{l}{$\overline{\AUC}_{\mathbf{t}}$} & 0.8246 & 0.0431 & 0.1735 & 0.1730 & 0.2923& \\
      \midrule
      \multicolumn{1}{l}{$\AUC_{\mathbf{t}}$} & 0.6674 & 0.0271 & 0 & 0.0271 & 1.0000 & \\
      \multicolumn{1}{l}{$\widehat{\AUC}_{\mathbf{t}}^{_{\left( \ast\right) }}$} & 0.6357 & 0.0654 & 0.0690 & 0.0727 & 0.3534 & \\
      \multicolumn{1}{l}{$\widehat{\AUC}_{\mathbf{t}}^{_{\left( .632\right) }}$} & 0.6995 & 0.0556 & 0.0615 & 0.0642 & 0.3570 & 50\\
      \multicolumn{1}{l}{$\widehat{\AUC}_{\mathbf{t}}^{_{\left( .632+\right)}}$} & 0.6685 & 0.0690 & 0.0646 & 0.0690 & 0.3522 & \\
      \multicolumn{1}{l}{$\overline{\AUC}_{\mathbf{t}}$} & 0.8091 & 0.0406 & 0.1473 & 0.1474 & 0.3517& \\
      \midrule
      \multicolumn{1}{l}{$\AUC_{\mathbf{t}}$} & 0.6808 & 0.0217 & 0 & 0.0217 &1.0000 & \\
      \multicolumn{1}{l}{$\widehat{\AUC}_{\mathbf{t}}^{_{\left( \ast\right) }}$} & 0.6533 & 0.0546 & 0.0602 & 0.0611 & 0.2451 & \\
      \multicolumn{1}{l}{$\widehat{\AUC}_{\mathbf{t}}^{_{\left( .632\right) }}$} & 0.7053 & 0.0471 & 0.0527 & 0.0531 & 0.2488 & 66\\
      \multicolumn{1}{l}{$\widehat{\AUC}_{\mathbf{t}}^{_{\left( .632+\right)}}$} & 0.6840 & 0.0568 & 0.0556 & 0.0569 & 0.2477 & \\
      \multicolumn{1}{l}{$\overline{\AUC}_{\mathbf{t}}$} & 0.7946 & 0.0355 & 0.1195 & 0.1192 & 0.2499& \\
      \midrule
      \multicolumn{1}{l}{$\AUC_{\mathbf{t}}$} & 0.6965 & 0.0158 & 0 & 0.0158 & 1.0000 & \\
      \multicolumn{1}{l}{$\widehat{\AUC}_{\mathbf{t}}^{_{\left( \ast\right) }}$} & 0.6738 & 0.0454 & 0.0483 & 0.0507 & 0.3422 & \\
      \multicolumn{1}{l}{$\widehat{\AUC}_{\mathbf{t}}^{_{\left( .632\right) }}$} & 0.7119 & 0.0399 & 0.0405 & 0.0428 & 0.3492 & 100\\
      \multicolumn{1}{l}{$\widehat{\AUC}_{\mathbf{t}}^{_{\left( .632+\right)}}$} & 0.7004 & 0.0452 & 0.0426 & 0.0453 & 0.3448 & \\
      \multicolumn{1}{l}{$\overline{\AUC}_{\mathbf{t}}$} & 0.7772 & 0.0312 & 0.0860 & 0.0866 & 0.3596& \\
      \midrule
      \multicolumn{1}{l}{$\AUC_{\mathbf{t}}$} & 0.7141 & 0.0090 & 0 & 0.0090 &1.0000 & \\
      \multicolumn{1}{l}{$\widehat{\AUC}_{\mathbf{t}}^{_{\left( \ast\right) }}$} & 0.6991 & 0.0298 & 0.0327 & 0.0334 & 0.2288 & \\
      \multicolumn{1}{l}{$\widehat{\AUC}_{\mathbf{t}}^{_{\left( .632\right) }}$} & 0.7205 & 0.0272 & 0.0273 & 0.0279 & 0.2291 & 200\\
      \multicolumn{1}{l}{$\widehat{\AUC}_{\mathbf{t}}^{_{\left( .632+\right)}}$} & 0.7170 & 0.0285 & 0.0279 & 0.0286 & 0.2294 & \\
      \multicolumn{1}{l}{$\overline{\AUC}_{\mathbf{t}}$} & 0.7573 & 0.0228 & 0.0487 & 0.0489 & 0.2277 & \\
      \bottomrule
    \end{tabular}

}
  \caption{Comparison of the different bootstrap-based estimators of the $\AUC$. they are comparable to each other in the RMS sense,
    $\widehat{\AUC}_{\mathbf{t}}^{_{\left(.632+\right)}}$ is almost unbiased, and all are weakly correlated with the true conditional
    performance $\AUC_{\mathbf{t}}$.}\label{tab1}
\end{table}

\begin{table}[t] \centering
  \begin{tabular}[c]{cc}\toprule
    \textbf{Estimator} & \textbf{Average RMS}\\\midrule
    \multicolumn{1}{l}{$\AUC_{\mathbf{t}}$} & $0$\\
    \multicolumn{1}{l}{$\widehat{\AUC}_{\mathbf{t}}^{_{\left(  *\right)  }}$} & $.07347$\\
    \multicolumn{1}{l}{$\widehat{\AUC}_{\mathbf{t}}^{_{(.632)}}$} & $.07409$\\
    \multicolumn{1}{l}{$\widehat{\AUC}_{\mathbf{t}}^{_{(.632+)}}$} & $.06735$\\
    \multicolumn{1}{l}{$\overline{\AUC}_{\mathbf{t}}$} & $.17808$\\\bottomrule
  \end{tabular}
  \caption{Average of RMS error of each estimator over the 10 experiments displayed in~\tablename~\ref{tab1}. The estimator
    $\widehat{\AUC}_{\mathbf{t}}^{\left( *\right) }$ is the next to $\widehat{\AUC}_{\mathbf{t}}^{\left( .632+\right) }$ with only 9\%
    increase in RMS.}\label{tab:AUC-Est-Average}
\end{table}%
We carried out different experiments to compare the three bootstrap-based estimators $\widehat{\AUC}_{\mathbf{t}}^{_{\left( \ast\right)
  }}$, $\widehat{\AUC}_{\mathbf{t}}^{_{\left( .632\right) }}$, and $\widehat {\AUC}_{\mathbf{t}}^{_{\left( .632+\right) }}$ considering
different dimensionalities, different parameter values, and training set sizes. All experiments provided consistent and similar
results. Here, in this section, we illustrate the results when the dimensionality $p=5$, for multinormal 2-class data, with $\Sigma_1 =
\Sigma_2 = \mathbf{I}$, $\mu_1 = \mathbf{0}$, $\mu_2 = c \mathbf{1}$, and $c$ is an adjusting parameter to adjust the Mahalanobis
distance $\Delta=\left[ (\mu_{1}-\mu_{2})'\Sigma^{-1}(\mu_{1}-\mu_{2})\right]^{1/2} = c^{2}p$. We adjust $c$ to keep a reasonable
inter-class separation of $\Delta = 0.8$. When the classifier is trained, it will be tested on a pseudo-infinite test set, here 1000
cases per class, to obtain a very good approximation to the true AUC for the classifier trained on this very training dataset; this is
called a single realization or a Monte-Carlo (MC) trial. Many realizations of the training datasets with same $n$ are generated over MC
simulation to study the mean and variance of the AUC for the Bayes classifier under this training set size. The number of MC trials is
1000 and the number of bootstraps is 100. It is apparent from~\figurename~\ref{fig4IEEE} that the $\widehat
{\AUC}_{\mathbf{t}}^{_{\left( \ast\right) }}$ is downward biased. This is a natural opposite of the upward bias observed in
\cite{Efron1997ImprovementsOnCross} when the performance measure was the true error rate as a measure of incorrectness, by contrast
with the true AUC as a measure of correctness. The $\widehat{\AUC}_{\mathbf{t}}^{_{\left( {.632}\right) }}$ is designed as a correction
for $\widehat{\AUC}_{\mathbf{t}}^{_{\left( \ast\right) }}$; it appears in the figure to correct for that but with an over-shoot. The
correct adjustment for the remaining bias is almost achieved by the estimator $\widehat{\AUC}_{\mathbf{t}}^{_{\left( {.632+}\right)
  }}$. The $\widehat{\AUC}_{\mathbf{t}}^{_{\left( {.632}\right) }}$ estimator can be seen as an attempt to balance between the two
extreme biased estimators, $\widehat{\AUC}_{\mathbf{t}}^{_{\left( \ast\right) }}$ and $\overline {\AUC}_{\mathbf{t}}$. However, it is
expected that the component of $\overline {\AUC}_{\mathbf{t}}$ that is inherent in both $\widehat{\AUC}_{\mathbf{t}}^{_{\left(
      {.632+}\right) }}$ and $\widehat{\AUC}_{\mathbf{t}}^{_{\left( {.632}\right) }}$ increases the variance of these two estimators
that my compensate for the decrease in the bias. Therefore, we assess all estimators in terms of the RMS, the root of the MSE defined
in~\eqref{eq104}, and report the results in~\tablename~\ref{tab1}. In addition, we average the RMS of these estimators over the 10
experiments of~\tablename~\ref{tab1} and list the average in~\tablename~\ref{tab:AUC-Est-Average}. It is evident that the $.632+$ is
slightly the overall winner with only 9\% decrease in RMS if compared to the $\widehat{\AUC}_{\mathbf{t}}^{_{\left( \ast\right) }}$
estimator. This almost agrees with the same result obtained for the error rate estimators and reported in~\tablename~\ref{tab2}.

\bigskip

In addition to the RMS,~\tablename~\ref{tab1} compares the estimators in terms of the RMS around mean ($\RMS_{AM}$): the root of the
mean squared difference between an estimate and the mean performance (the mean over all possible training sets), instead of the
conditional performance (conditional on a particular training set). The motivation behind that is explained next. The estimators
$\widehat{\AUC}_{\mathbf{t}}^{_{\left(*\right) }}$ and $\widehat{\AUC}_{\mathbf{t}}^{_{\left(1,1\right) }}$ seem, at least from their
formalization, to estimate the mean AUC of the classifier (this is the analogue of $\widehat{\Err}_{\mathbf{t}}^{_{\left(*\right) }}$
and $\widehat{\Err}_{\mathbf{t}}^{_{\left(1\right) }}$). However, the basic motivation for the $\widehat{\AUC}_{\mathbf{t}}^{_{\left(
      .{632}\right) }}$ and $\widehat{\AUC}_{\mathbf{t}}^{_{\left( .{632+}\right) }}$ is to estimate the AUC conditional on the given
dataset $\mathbf{t}$ (this is the analogue of $\widehat{\Err}_{\mathbf{t}}^{_{\left( .{632}\right) }}$ and
$\widehat{\Err}_{\mathbf{t}}^{_{\left( .{632+}\right) }}$). Nevertheless, as mentioned in~\cite{Efron1997ImprovementsOnCross} and
detailed in~\cite{Zhang1995AssessingPrediction} the CV, the basic ingredient of the bootstrap based estimators, is weakly correlated
with the true performance on a sample by sample basis. This means that no estimator has a preference in estimating the conditional
performance. Section~\ref{sec:WeakCorr} elaborates more on this phenomenon.

\subsection{Components of Variance and Weak Correlation}\label{sec:WeakCorr}
\begin{figure}[t]\centering
  \includegraphics[height=2in]{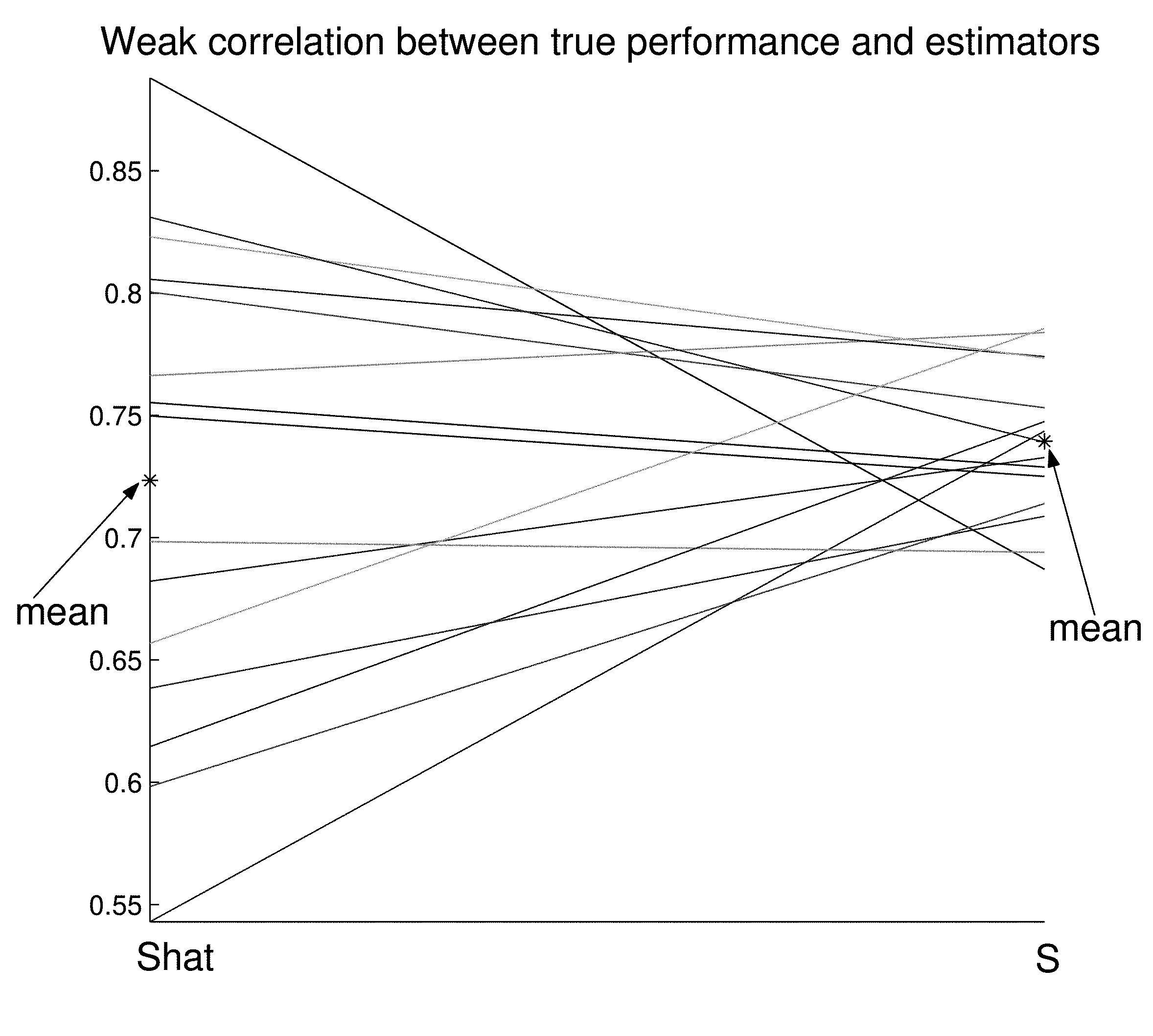}
  \caption{The lack of correlation (or the weak correlation) between the bootstrap-based estimators and the true conditional
    performance. Every line connects the true performance of the classifier trained on a data set $\mathbf{t}_{i}$ and the estimated
    value. The figure represents 15 trials of the 1000 MC trials. Two nearby values of true performance may correspond to two widely
    separated estimates on different sides of the mean.}\label{fig_WeakCorr}
\end{figure}
Many simulation results, e.g.,~\cite{Efron1983EstimatingTheError,Efron1997ImprovementsOnCross}, show that there is only a weak
correlation between the CV estimator and the conditional true error rate $\Err_{\mathbf{t}}$. This issue is discussed in mathematical
detail in the excellent paper by~\cite{Zhang1995AssessingPrediction}, which therefore concludes that the CV estimator should not be
used to estimate the true error rate of a classification rule conditional on a particular training data set. Other estimators discussed
in the present article have this same attribute, since they have the same resampling ingredient of the CV estimator and ``\textit{we
  would guess, for any other estimate of conditional prediction error}''~\citep[Sec. 7.12,][]{Hastie2009ElemStat}. We provide our
simple mathematical elaboration as follows. Denote the true performance of the classification rule conditional on the training set
$\mathbf{t}$ (whether $\Err_{\mathbf{t}}$, $\AUC_{\mathbf{t}}$, or any other performance measure) by $S_{\mathbf{t}}$, the
unconditional performance by $\MEAN_{\mathbf{t}}S_{\mathbf{t}}$, and an estimator of either of them by $\widehat{S}_{\mathbf{t}}$. For
easier notation we can unambiguously drop the subscript $\mathbf{t}$ and decompose the MSE as%
\begin{subequations}
  \begin{align}
    \MSE(\widehat{S}, S)  &  = \MEAN(\widehat{S} - S)^2 \\
                          & = \MEAN(\widehat{S} - \MEAN S)^2 + \Var(S) - 2\Cov(\widehat{S},S).
  \end{align}%
\end{subequations}
Then, by normalizing with the standard deviations we get:
\begin{align}
  \frac{\MSE(\widehat{S}, S)}{\sigma_S \sigma_{\widehat{S}}} &= \frac{\MSE(\widehat{S}, \MEAN S)}{\sigma_S \sigma_{\widehat{S}}} + \frac{\sigma_S}{\sigma_{\widehat{S}}} - 2\rho_{\widehat{S}S}.\label{eq:14}
\end{align}
This equation relates four crucial components to each other:
\begin{itemize}
\item $\MSE(\widehat{S}, S)\bigl/\sigma_S \sigma_{\widehat{S}}$, the normalized MSE of $\widehat{S}$, if we see it as an estimator of the
  conditional performance $S$.

\item $\MSE(\widehat{S}, \MEAN S)\bigl/\sigma_S \sigma_{\widehat{S}}$, the normalized MSE of $\widehat{S}$, if we see it as an estimator
  of the expected performance $\MEAN S$ (and therefore called MSE around the mean).

\item $\sigma_S \bigl/ \sigma_{\widehat{S}}$, the standard deviation ratio between $S$ and $\widehat{S}$.

\item $\rho_{\widehat{S}S}$, the correlation coefficient between $S$ and $\widehat{S}$.
\end{itemize}
From~\eqref{eq:14}, an estimator $\widehat{S}$ is a good candidate to estimate $S$ than $\MEAN S$ if its $\MSE(\widehat{S}, S)$ is less
than its $\MSE(\widehat{S}, \MEAN S)$. Then, it is the responsibility of the correlation coefficient $\rho_{\widehat{S}S}$ to be high
enough to cancel $\sigma_S \bigl/ \sigma_{\widehat{S}}$ and a portion of $\MSE(\widehat{S}, \MEAN S)$. Unfortunately, this is not the
case as we illustrate experimentally in~\tablename~\ref{tab1}, which provides all quantities of the decomposition~\eqref{eq:14}. It is
obvious from the values that $\RMS(\widehat{S}, S)$ and $\RMS(\widehat{S}, \MEAN S)$ are very close to each other because the quantity
$\sigma_S \bigl/ \sigma_{\widehat{S}} - 2\rho_{\widehat{S}S} \simeq 0.413 - 2 \times 0.290 = -0.167$ (on average over the 10
experiments shown in the table). Moreover, in some cases, e.g., the first experiment, it goes as low as $-0.052$. The correlation
between $\widehat{S}$ and $S$ is weak to cast $\widehat{S}$ as an estimate to $S$, although it is designed to estimate it! For more
illustration,~\figurename~\ref{fig_WeakCorr} visualizes the components in Eq.~\eqref{eq:14} and the numbers in~\tablename~\ref{tab1}.
This figure shows 15 realizations of the 1000 MC trials of the same experiment above. On the right, are the true values of $S$ when
trained on these different 15 training sets. On the left, are the corresponding 15 estimated values of $\widehat{S}$. The lines provide
links between the true values and the corresponding estimates. This figure shows that two nearby true values of $S$ are likely to have
two widely separated estimated values $\widehat{S}$ on different sides of the mean. This visually illustrates the lack of correlation
(or the weak correlation) between the estimators and the true conditional performanc.

\subsection{Two Competing Classifiers}\label{sec:two-comp-class}
\begin{table}[t]\centering
  \begin{tabular}{lccc}
    \toprule
    \textbf{Metric} $M$ & \textbf{LDA} & \textbf{QDA} & $\Delta$\\
    \midrule
    $\MEAN\  M_{\mathbf{t}}$ & $.7706$ & $.7163$ & $.0543$\\
    $\SD\ M_{\mathbf{t}}$ & $.0313$ & $.0442$ & $.0343$\\
    $\MEAN\  \widehat{M}^{(1,1)}$ & $.7437$ & $.6679$ & $.0758$\\
    \midrule
    $\SD\ \widehat{M}^{(1,1)}$ & $.0879$ & $.0944$ &$.0533$\\
    \midrule
    $\MEAN\ \widehat{\SD}\ {{\widehat{M}}^{(1,1)}}$ &$.0898$ & $.1003$ & $.0708$\\
    \midrule
    $\SD\ \widehat{\SD}\ \widehat{M}^{(1,1)}$ &$.0192$ & $.0163$ & $.0228$\\
    \bottomrule
  \end{tabular}
  \caption{Estimating the uncertainty in the estimator that estimates the difference in performance of two competing classifiers, the
    LDA and the QDA. The quantity $M$ represents $\AUC_1$ for LDA, $\AUC_2$ for QDA, and $\Delta$ for the
    difference.}\label{Fig_DifAUC}
\end{table}
If the assessment problem is how to compare two classifiers, rather than the individual performance, then the measure to be used is
either the conditional difference%
\begin{equation}
  \Delta_{\mathbf{t}}=\AUC_{1_{\mathbf{t}}}-\AUC_{2_{\mathbf{t}}},
\end{equation}
or the mean, unconditional, difference%
\begin{equation}
  \Delta=\operatorname*{E}\Delta_{\mathbf{t}}=\operatorname*{E}\left[ \AUC_{1_{\mathbf{t}}}-\AUC_{2_{\mathbf{t}}}\right],
\end{equation}
where, we defined them for the AUC just for illustration with immediate identical treatment for other measures. Then it is obvious that
there is nothing new in the estimation task, i.e., it is merely the difference of the performance estimate of each classifier, i.e.,%
\begin{equation}
  \widehat{\Delta}=\widehat{\operatorname*{E}\AUC_{1_{\mathbf{t}}}}-\widehat{\operatorname*{E}\AUC_{2_{\mathbf{t}}}}, \label{EQ_DifAUC}
\end{equation}
where each of the two estimators in (\ref{EQ_DifAUC}) is obtained by any estimator. A natural candidate, from the point of view of the
present chapter is the LPOB estimator $\widehat{\AUC}^{_{\left( 1,1\right) }}$---because of both the smoothness and weak correlation
issues discussed so far.

Then, how to estimate the uncertainty (variance) of $\widehat{\Delta}$. This is very similar to estimating the variance in
$\widehat{\operatorname*{E}\AUC_{\mathbf{t}}}$. There is nothing new in estimating $\operatorname*{Var}\widehat{\Delta}$. It is
obtained by replacing $\widehat{\AUC}^{_{\left( 1,1\right) }}$, in~\cite{Yousef2005EstimatingThe}, by the statistic $\widehat{\Delta}$
in (\ref{EQ_DifAUC}). For demonstration, typical values are given in~\tablename~\ref{Fig_DifAUC}, for comparing the linear and
quadratic discriminants, where the training set size per class is $20$ and number of features is $4$.


  \section{Discussion and Conclusion}\label{sec:conclusion}
In this chapter, the very important topic of the assessment of ML algorithms is reviewed, with an emphasis on the nonparametric
assessment of classification rules. The topic is quite important to many fields and applications, in particular cyberphysical security,
where ML algorithms are almost ubiquitous. We started with reviewing the basic nonparametric methods for estimating the bias and
variance of a statistic. Then, we reviewed the basic resampling-based methods for estimating the error rate of a classification rule.
Departing from that, we extended these estimators from estimating the error rate (a one-sample statistic) to estimating the AUC (a
two-sample statistic). This extension is theoretically justified, and not just an ad hoc application. Among these estimators, we
identified those that are smooth and eligible for estimating their standard error using the IF method.

\bigskip

It was interesting to see, through the whole chapter, the connection among different resampling-based estimators. It is worth
mentioning that, in addition to the conventional $K$-fold CV, there are other versions and variants, which are usually used in an ad
hoc way by many practitioners. The formalization of these versions and variants, and the mathematical connection among them, along with
their connection to the bootstrap-based estimators, all can be established in the same spirit and approach followed in the present
chapter. However, many of them are unsmooth except possibly the repeated CV, which is partially smooth and suitable for the IF
method~\citep{Yousef2019LeisurelyLookVersionsVariants-arxiv,Yousef2021EstimatingStandardErrorCross}.

\bigskip

With this rich variety of estimators, a practitioner may legitimately wonder about the ``optimal'' estimator (in terms of any
optimality criterion) that should be systematically used. There are three aspects, on which we can base our comparison: accuracy,
uncertainty estimation, and computational efficiency.

In terms of accuracy, it is surprising to know that, from the few number of comparative studies available in the literature, there is
no overall winner among these estimators. All of them have comparable accuracy, measured in terms of RMS, with a little superiority of
the .632+ bootstrap estimator. In addition, and most importantly, all estimators have a weak correlation with the true conditional
performance (e.g., $\Err_{\mathbf{t}}$, the conditional error rate, or $\AUC_{\mathbf{t}}$, the conditional AUC), a phenomenon that
allows them to be eligible only for estimating the mean true performance (e.g., $\MEAN_{\mathbf{t}}\Err_{\mathbf{t}}$ or
$\MEAN_{\mathbf{t}}\AUC_{\mathbf{t}}$), where the mean is taken over the population of training datasets as explained through the
chapter. Said differently, the performance estimation that a practitioner obtains using, e.g., the CV, is not an estimation of the
performance of this very trained ML algorithm; rather, it is an estimation of the mean performance of this algorithm had we trained it
on all possible training datasets of the same size! We quote from~\citep[Sec. 7.12]{Hastie2009ElemStat}:
\begin{quotation}
  ``\textit{This phenomenon also occurs for bootstrap estimates of error, and we would guess, for any other estimate of conditional
    prediction error}.''.
\end{quotation}

In terms of the variance estimation of these estimators (not the estimation of the variance of the algorithm itself), only a few of
them are smooth and candidates for a sophisticated method like the IF. The ordinary $K$-fold CV is not among those! Rather, only the
computationally expensive version of it, the repeated CV, is partially smooth as mentioned above.

It terms of the computational aspects, the bootstrap-based estimators are computationally expensive. If compared to the conventional
$K$-fold CV, which requires only $K$ iterations of both training and testing, the former require hundreds of bootstrap replications.
Because the majority of recent ML applications involve both massive datasets and complex algorithms, including DNN that is very
computationally expensive, it is obvious that the CV may be more practical than the bootstrap-based estimators. However, for some other
fields, e.g., cyberphysical security, many applications produce tabular (structured) data. Tabular data are more suitable for the
traditional and less computationally expensive ML algorithms. Therefore, serious practitioners in these fields and applications may
need to keep all of these estimators in their toolbox. Moreover, it is quite prudent to see a future benchmark that compiles these
estimators, along with different datasets from a wide range of applications, in a single comprehensive comparative study.


  \section{Acknowledgment}\label{sec:acknoledgment}
The author is grateful to the U.S. Food and Drug Administration (FDA) for funding a very early stage of this chapter, and to Dr. Kyle
Myers for her support. In his memorial, special thanks and gratitude to Dr. Robert F. Wagner, the supervisor and the teacher, or Bob
Wagner, the big brother and friend. He reviewed a very early version of this chapter before he passed away.


  \section{Appendix}\label{sec:appendix-1}
\subsection{Proofs}\label{sec:appendix}
\begin{lemma}\label{lem:perf-estim-from}The maximum likelihood estimation (MLE) for the probability mass function under nonparametric
  distribution, given a sample of $n$ observations, is given by:%
  \begin{equation}
    \hat{F}:\ \text{mass}~\frac{1}{n}~\text{on}\,~t_i,\quad i=1,\ldots,n.\label{eq52-01}
  \end{equation}
\end{lemma}
\begin{proof}The proof is carried out by maximizing the likelihood function $l(f)=\prod\limits_{i=1}^{n}{p_{i}}$, which can be rewritten
  under the constraint $\sum_{i}p_{i}=1$, using a Lagrange's multiplier, as:%
  \begin{equation}
    l(f)=\prod\limits_{i=1}^{n}{p_{i}}+\lambda\left(  \sum\limits_{i=1}^{n}{p_{i}}-1\right).\label{eq49}%
  \end{equation}
  The likelihood (\ref{eq49}) is maximized by taking the first derivative and setting it to zero to obtain:%
  \begin{equation}
    \frac{\partial l(f)}{\partial p_{j}}=\prod\limits_{i\neq j}{p_{i}}+\lambda\overset{set}{=}0,\quad j=1,\ldots,n.\label{eq50}%
  \end{equation}
  These $n$ equations along with the constraint $\sum_{i}p_{i}=1$ can be solved straightforwardly to give $\hat{p}_{i}=\frac{1}{n},\
  i=1,\ldots,n$, which completes the proof.\quad\qed
\end{proof}

\begin{lemma} The \textit{no-information} $\AUC$ is given by $\gamma_{\AUC} = 0.5$.\label{lem:NoInfo}
\end{lemma}
\begin{proof}
  $\gamma_{\AUC}$, an analogue to the \textit{no-information error rate} $\gamma$, is given by (\ref{eq47}) but with TPF and FPF given
  under the \textit{no-information} distribution $\MEAN_{0F}$ (see Sec.~\ref{para:mylabel4}). Therefore, assume that there are $n_{1}$
  and $n_2$ observations from class $\omega_{1}$ and $\omega_{2}$, respectively. Assume also for a fixed threshold $th$ the two
  quantities that define the error rate are $\TPF$ and $\FPF$. Also, assume that the sample observations are tested by the classifier
  and each sample has been assigned a decision value (score). Under the \textit{no-information} distribution, consider the following.
  For every decision value $h_{\mathbf{t}}(x_{i})$ assigned for the observation $t_{i}=(x_{i},y_{i})$, create new $n_{1}+n_{2}-1$
  observations; all of them have the same decision value $h_{\mathbf{t}}(x_{i})$, while their responses are equal to the responses of
  the rest $n_{1}+n_{2}-1$ observations $t_{j},\ j \neq i$. Under this new sample that consists of $(n_{1}+n_{2})^{2}$ observations, it
  is quite easy to see that the new TPF and FPF for the same threshold $th$ are given by
  $\FPF_{0\widehat{F},th}=\TPF_{0\widehat{F},th}= (\TPF\cdot n_{1}+\FPF\cdot n_{2})/(n_{1}+n_{2})$. This means that the ROC curve under
  the \textit{no-information} rate is a straight line with slope equal to one; this directly gives $\gamma_{\AUC} = 0.5$.
\end{proof}

\subsection{More On Influence Function (IF)}\label{sec:influence-function}
Assume that there is a distribution $G$ near to the distribution $F$; then under some regularity conditions~\citep[see, e.g.,][Ch.
2]{Huber1996RobustStatistical} a functional $s$ can be approximated as:%
\begin{equation}
  s(G)\approx s(F)+\int{IC_{s,F}(x)~dG(x)}.\label{eq73}%
\end{equation}
The residual error can be neglected since it is of a small order in probability. Some properties of (\ref{eq73}) are:%
\begin{equation}
  \int{IC_{T,F}(x)~dF(x)=0},\label{eq74}
\end{equation}
and the asymptotic variance of $s(F)$ under $F$, following from (\ref{eq74}), is given by:%
\begin{equation}
  \Var_{F}s(F)\simeq\int{\left[{IC_{T,F}(x)}\right]^{2}~dF(x)},\label{eq75}
\end{equation}
which can be considered as an approximation to the variance under a distribution $G$ near to $F$. Now, assume that the functional $s$
is a functional statistic in the dataset $\mathbf{x}=\{x_{i}:x_{i}\sim F,\ i=1,\ldots,n\}$. In that case the influence curve
(\ref{eq71}) is defined for each sample case $x_{i}$, under the true distribution $F$ as:%
\begin{equation}
  U_{i}(s,F)=\lim_{\varepsilon\rightarrow0}\frac{s(F_{\varepsilon,i})-s(F)}{\varepsilon}=\left.  {\frac{\partial s(F_{\varepsilon,i})}{\partial\varepsilon}}\right\vert _{\varepsilon=0},\label{eq76}
\end{equation}
where $F_{\varepsilon,i}$ is the distribution under the perturbation at observation $x_{i}$. Eq. (\ref{eq76}) is called the IF. If the
distribution $F$ is not known, the MLE $\hat{F}$ of the distribution $F$ is given by (\ref{eq52}), and as an approximation $\hat{F}$
may substitute for $F$ in (\ref{eq76}). The result may then be called the empirical IF \citep{Mallows1974OnSomeTopics}, or
infinitesimal jackknife~\citep{Jaeckel1972TheInfinitesimal}. In such an approximation, the perturbation defined in (\ref{eq70}) can be
rewritten as:%
\begin{equation}
  \hat{F}_{\varepsilon,i}=(1-\varepsilon)\hat{F}+\varepsilon\delta_{x_{i}},~~x_{i}\in\mathbf{x},~~i=1,\ldots,n.\label{eq77}
\end{equation}
This kind of perturbation is illustrated in~\figurename~\ref{fig6}.
\begin{figure}[t]\centering
  \includegraphics[width=0.75\textwidth]{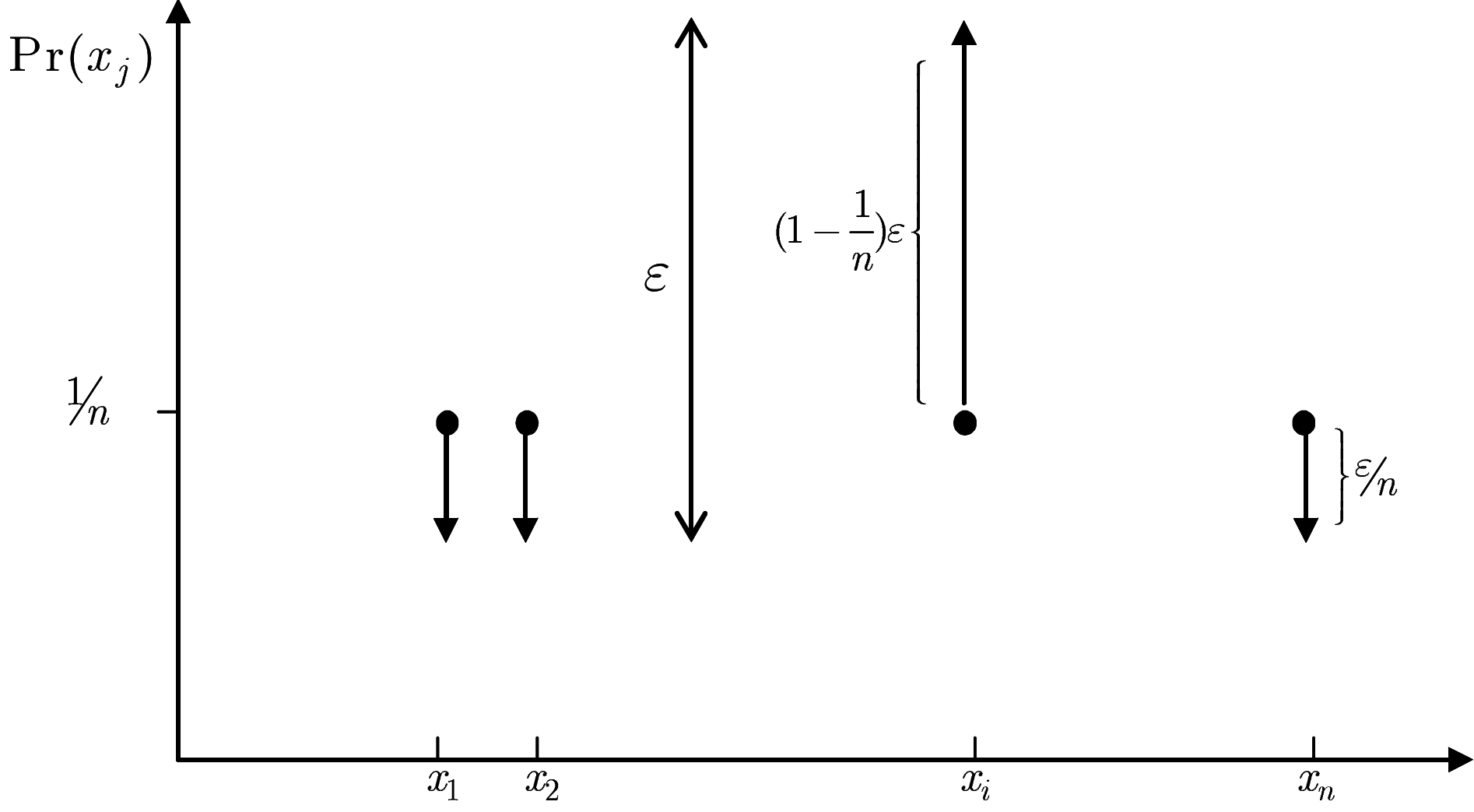}
  \caption{The new probability masses for the dataset $\mathbf{x}$ under a perturbation at sample case $x_{i}$ obtained by letting the
    new probability, at $x_{i}$ exceed the new probability at any other case $x_{i}$ by, $\varepsilon$.}\label{fig6}
\end{figure}
It will often be useful to write the probability mass function of (\ref{eq77}) as:%
\begin{equation}
  \hat{f}_{\varepsilon,i}(x_{j})=\left\{  {%
      \begin{array}[c]{lll}
        \frac{1-\varepsilon}{n}+\varepsilon &  & j=i\\
        \frac{1-\varepsilon}{n} &  & j\neq i
      \end{array}
    }\right..\label{eq78}%
\end{equation}
A very interesting case arises from (\ref{eq78}) if $-1/(n+1)$ is substituted for $\varepsilon$. In this case the new probability mass
assigned to the point $x_{j=i}$ in (\ref{eq78}) will be zero. This value of $\varepsilon$ simply generates the jackknife estimate
discussed in Sec.~\ref{subsubsec:jackknife}, where the whole observation is removed from the dataset.

Substituting $\hat{F}$ for $G$ in (\ref{eq73}) and combining the result with (\ref{eq76}) gives the IF approximation
for any functional statistic under the empirical distribution $\hat{F}$. The result is:%
\begin{subequations}
  \begin{align}
    s(\hat{F})  &  =s(F)+\frac{1}{n}\sum\limits_{i=1}^{n}{U_{i}}(s,F)+O_{p}(n^{-1})\label{eq79}\\
                &  \approx s(F)+\frac{1}{n}\sum\limits_{i=1}^{n}{U_{i}(s,F)}.
  \end{align}
\end{subequations}
The term $O_{p}(n^{-1})$ reads \textquotedblleft big-O of order $1/n$ in probability\textquotedblright. In general, $U_{n}=O_p(d_{n})$ if
$U_{n}/d_{n}$ is bounded in probability, i.e., $\Pr\{|U_{n}|/d_{n}<k_{\varepsilon }\}>1-\varepsilon$ $\forall$ $\varepsilon>0$. This
concept can be found in~\citet[Ch. 2]{BarndorffNielsen1989Asymptotic}. Then the asymptotic variance expressed in (\ref{eq75}) can be
given for $s(F)$ by:%
\begin{equation}
  \Var_{F} s =\frac{1}{n}\MEAN_{F} U^{2}(x_{i},F),\label{eq80}
\end{equation}
which can be approximated under the empirical distribution $\hat{F}$ to give the nonparametric estimate of the variance for a statistic
$s$ by:%
\begin{equation}
  \widehat{\Var}_{\hat{F}} s=\frac{1}{n^{2}}\sum\limits_{i=1}^{n}{U_{i}^{2}(x_{i},\hat{F})}.\label{eq81}
\end{equation}

\subsection{ML In Other Fields}\label{sec:some-hist-notes}
In this section we provide very brief miscellanea from other fields for the reader to see a bigger picture of this chapter. As already
was mentioned, ML is crucial to many applications. For example, in the medical imaging field, a tumor on a mammogram must be classified
as malignant or benign. This is an example of prediction, regardless of whether it is done by a radiologist or by a computer aided
detection (CAD) software. In either case, the prediction is done based on learning from previous mammograms. The features, i.e.,
predictors, in this case may be the size of the tumor, its density, various shape parameters, etc. The output, i.e., response, is
categorical and belongs to the set: $\mathcal{G} =\{benign,\ malignant\}$. There are so many such examples in biology and
medicine that it is almost a field unto itself, i.e., biostatistics. The task may be diagnostic as in the mammographic example, or
prognostic where, for example, one estimates the probability of occurrence of a second heart attack for a particular patient who has
had a previous one. All of these examples involve a prediction step based on previous learning. A wide range of commercial and military
applications arises in the field of satellite imaging. Predictors in this case can be measures from the image spectrum, while the
response can be the type of land, crop, or vegetation of which the image was taken.

\bigskip

Some expressions and terminology of ML belong to some fields and applications more than the others. E.g., it is conventional in medical
imaging to refer to $e_{1}$ as the false negative fraction (FNF), and $e_{2}$ as the false positive fraction (FPF). This is because
diseased patients typically have a higher output value for a test than non-diseased patients. For example, a patient belonging to class
1 whose test output value is less than the threshold setting for the test will be called ``test negative'', while the patient is in
fact in the diseased class. This is a false negative decision; hence the name FNF. The situation is reversed for the other error
component.

\bigskip

The importance of the AUC is natural and unquestionable in some applications than others. The equivalence of the area under the
empirical ROC and the Mann-Whitney-Wilcoxon statistic is the basis of its use in the assessment of diagnostic tests;
see~\cite{Hanley1982TheMeaning}.~\cite{Swets1986IndicesOfDiscrimination} has recommended it as a natural summary measure of detection
accuracy on the basis of signal-detection theory. Applications of this measure are widespread in the literature on both human diagnosis
and computer-aided diagnosis, in medical imaging \citep{Jiang1999ImprovingBreast}. In the field of machine
learning,~\cite{Bradley1997TheUseOfTheAreaUnder} has recommended it as the preferred summary measure of accuracy when a single number
is desired. These references also provide general background and access to the large literature on the subject.

\bigskip

Even the mistakes committed by some practitioners are obvious in some fields more than others. E.g., in DNA microarrays, these mistakes
are fatal and produce very fragile results. This is because of the very high dimensionality of the problem with respect to the amount
of available dataset. A more elaborate assessment phase should follow the design and construction phase in such ill-posed applications.


  \putbib[booksIhave,publications]

\end{bibunit}



\begin{thebibliography}{35}
\providecommand{\natexlab}[1]{#1}
\providecommand{\url}[1]{{#1}}
\providecommand{\urlprefix}{URL }
\expandafter\ifx\csname urlstyle\endcsname\relax
  \providecommand{\doi}[1]{DOI~\discretionary{}{}{}#1}\else
  \providecommand{\doi}{DOI~\discretionary{}{}{}\begingroup
  \urlstyle{rm}\Url}\fi
\providecommand{\eprint}[2][]{\url{#2}}

\bibitem[{Barndorff-Nielsen and Cox(1989)}]{BarndorffNielsen1989Asymptotic}
Barndorff-Nielsen OE, Cox DR (1989) {Asymptotic techniques for use in
  statistics}. Chapman and Hall, London; New York

\bibitem[{Bradley(1997)}]{Bradley1997TheUseOfTheAreaUnder}
Bradley AP (1997) {The Use of the Area Under the ROC Curve in the Evaluation of
  Machine Learning algorithms}. Pattern Recognition 30(7):1145

\bibitem[{Breiman et~al(1984)Breiman, Friedman, Olshen, and
  Stone}]{Breiman1984ClassificationAnd}
Breiman L, Friedman J, Olshen R, Stone C (1984) {Classification and regression
  trees}. Wadsworth International Group, Belmont, Calif.

\bibitem[{Chen et~al(2012)Chen, Gallas, and Yousef}]{Chen2012ClassVar}
Chen W, Gallas BD, Yousef WA (2012) {Classifier Variability: Accounting for
  Training and testing}. Pattern Recognition 45(7):2661--2671

\bibitem[{Efron(1979)}]{Efron1979BootstrapMethods}
Efron B (1979) {Bootstrap Methods: Another Look At the Jackknife}. The Annals
  of Statistics 7(1):1--26

\bibitem[{Efron(1981)}]{Efron1981NonparametricEstimates}
Efron B (1981) {Nonparametric Estimates of Standard Error: the Jackknife, the
  Bootstrap and Other Methods}. Biometrika 68(3):589--599

\bibitem[{Efron(1982)}]{Efron1982TheJackknife}
Efron B (1982) {The jackknife, the bootstrap, and other resampling plans}.
  Society for Industrial and Applied Mathematics, Philadelphia, Pa.

\bibitem[{Efron(1983)}]{Efron1983EstimatingTheError}
Efron B (1983) {Estimating the Error Rate of a Prediction Rule: Improvement on
  Cross-Validation}. Journal of the American Statistical Association
  78(382):316--331

\bibitem[{Efron(1986)}]{Efron1986HowBiasedIs}
Efron B (1986) {How Biased Is the Apparent Error Rate of a Prediction Rule?}
  Journal of the American Statistical Association 81(394):461--470

\bibitem[{Efron and Stein(1981)}]{Efron1981TheJacknifeEstimate}
Efron B, Stein C (1981) {The Jackknife Estimate of Variance}. The Annals of
  Statistics 9(3):586--596

\bibitem[{Efron and Tibshirani(1993)}]{Efron1993AnIntroduction}
Efron B, Tibshirani R (1993) {An introduction to the bootstrap}. Chapman and
  Hall, New York

\bibitem[{Efron and Tibshirani(1995)}]{Efron1995CrossValidation}
Efron B, Tibshirani R (1995) {Cross Validation and the Bootstrap: Estimating
  the Error Rate of a Prediction Rule}. Technical Report 176, Stanford
  University, Department of Statistics

\bibitem[{Efron and Tibshirani(1997)}]{Efron1997ImprovementsOnCross}
Efron B, Tibshirani R (1997) {Improvements on Cross-Validation: the $.632+$
  Bootstrap Method}. Journal of the American Statistical Association
  92(438):548--560

\bibitem[{Fukunaga(1990)}]{Fukunaga1990Introduction}
Fukunaga K (1990) {Introduction to statistical pattern recognition}, 2nd edn.
  Academic Press, Boston

\bibitem[{H\'{a}jek et~al(1999)H\'{a}jek, \v{S}id\'{a}k, and
  Sen}]{Hajek1999TheoryOfRank}
H\'{a}jek J, \v{S}id\'{a}k Z, Sen PK (1999) {Theory of rank tests}, 2nd edn.
  Academic Press, San Diego, Calif.

\bibitem[{Hampel(1974)}]{Hampel1974TheInfluence}
Hampel FR (1974) {The Influence Curve and Its Role in Robust Estimation}.
  Journal of the American Statistical Association 69(346):383--393

\bibitem[{Hampel(1986)}]{Hampel1986RobustStatistics}
Hampel FR (1986) {Robust statistics : the approach based on influence
  functions}. Wiley, New York

\bibitem[{Hanley(1989)}]{Hanley1989ROCMeth}
Hanley JA (1989) {Receiver Operating Characteristic (ROC) Methodology: the
  State of the art}. Critical Reviews in Diagnostic Imaging 29(3):307--335

\bibitem[{Hanley and McNeil(1982)}]{Hanley1982TheMeaning}
Hanley JA, McNeil BJ (1982) {The Meaning and Use of the Area Under a Receiver
  Operating Characteristic (ROC) curve}. Radiology 143(1):29--36

\bibitem[{Hastie et~al(2009)Hastie, Tibshirani, and
  Friedman}]{Hastie2009ElemStat}
Hastie T, Tibshirani R, Friedman JH (2009) {The elements of statistical
  learning: data mining, inference, and prediction}, 2nd edn. Springer, New
  York

\bibitem[{Huber(1996)}]{Huber1996RobustStatistical}
Huber PJ (1996) {Robust statistical procedures}, 2nd edn. Society for
  Industrial and Applied Mathematics, Philadelphia

\bibitem[{Jaeckel(1972)}]{Jaeckel1972TheInfinitesimal}
Jaeckel L (1972) {The Infinitesimal jackknife}. Memorandum, MM 72-1215-11, Bell
  Lab Murray Hill, NJ

\bibitem[{Jiang et~al(1999)Jiang, Nishikawa, Schmidt, Metz, Giger, and
  Doi}]{Jiang1999ImprovingBreast}
Jiang Y, Nishikawa RM, Schmidt RA, Metz CE, Giger ML, Doi K (1999) {Improving
  Breast Cancer Diagnosis With Computer-Aided diagnosis}. Academic Radiology
  6(1):22--33

\bibitem[{Mallows(1974)}]{Mallows1974OnSomeTopics}
Mallows C (1974) {On Some Topics in robustness}. Memorandum, MM 72-1215-11,
  Bell Lab Murray Hill, NJ

\bibitem[{Randles and Wolfe(1979)}]{Randles1979IntroductionTo}
Randles RH, Wolfe DA (1979) {Introduction to the theory of nonparametric
  statistics}. Wiley, New York

\bibitem[{Sahiner et~al(2001)Sahiner, Chan, Petrick, Hadjiiski, Paquerault, and
  Gurcan}]{Sahiner2001ResamplingSchemes}
Sahiner B, Chan HP, Petrick N, Hadjiiski L, Paquerault S, Gurcan MN (2001)
  {Resampling Schemes for Estimating the Accuracy of a Classifier Designed With
  a Limited Data Set}. Medical Image Perception Conference IX, Airlie
  Conference Center, Warrenton VA, 20-23

\bibitem[{Sahiner et~al(2008)Sahiner, Chan, and
  Hadjiiski}]{Sahiner2008ClassifierPerformance}
Sahiner B, Chan HP, Hadjiiski L (2008) {Classifier Performance Prediction for
  Computer-Aided Diagnosis Using a Limited dataset}. Medical Physics 35(4):1559

\bibitem[{Stone(1974)}]{Stone1974CrossValidatory}
Stone M (1974) {Cross-Validatory Choice and Assessment of Statistical
  Predictions}. Journal of the Royal Statistical Society Series B
  (Methodological) 36(2):111--147

\bibitem[{Swets(1986)}]{Swets1986IndicesOfDiscrimination}
Swets JA (1986) {Indices of Discrimination Or Diagnostic Accuracy: Their ROCs
  and Implied Models}. Psychological Bulletin 99:100--117

\bibitem[{Yousef(2019)}]{Yousef2019LeisurelyLookVersionsVariants-arxiv}
Yousef WA (2019) A leisurely look at versions and variants of the cross
  validation estimator. arXiv preprint arXiv:190713413

\bibitem[{Yousef(2021)}]{Yousef2021EstimatingStandardErrorCross}
Yousef WA (2021) Estimating the standard error of cross-validation-based
  estimators of classifier performance. Pattern Recognition Letters
  146:115--145

\bibitem[{Yousef et~al(2004)Yousef, Wagner, and Loew}]{Yousef2004ComparisonOf}
Yousef WA, Wagner RF, Loew MH (2004) {Comparison of Non-Parametric Methods for
  Assessing Classifier Performance in Terms of ROC Parameters}. In: Applied
  Imagery Pattern Recognition Workshop, 2004. Proceedings. 33rd; IEEE Computer
  Society, pp 190--195

\bibitem[{Yousef et~al(2005)Yousef, Wagner, and Loew}]{Yousef2005EstimatingThe}
Yousef WA, Wagner RF, Loew MH (2005) {Estimating the Uncertainty in the
  Estimated Mean Area Under the ROC Curve of a Classifier}. Pattern Recognition
  Letters 26(16):2600--2610

\bibitem[{Yousef et~al(2006)Yousef, Wagner, and Loew}]{Yousef2006AssessClass}
Yousef WA, Wagner RF, Loew MH (2006) {Assessing Classifiers From Two
  Independent Data Sets Using ROC Analysis: a Nonparametric Approach}. Pattern
  Analysis and Machine Intelligence, IEEE Transactions on 28(11):1809--1817

\bibitem[{Zhang(1995)}]{Zhang1995AssessingPrediction}
Zhang P (1995) {Assessing Prediction Error in Nonparametric Regression}.
  Scandinavian Journal Of Statistics 22(1):83--94

\end{thebibliography}
\end{document}